\documentclass[letterpaper]{article}
\usepackage{aaai2027}
\usepackage[hyphens]{url}
\usepackage{graphicx}
\usepackage{natbib}
\usepackage{caption}
\usepackage{booktabs}
\usepackage{amsmath}
\usepackage{amssymb}
\usepackage{amsthm}
\newtheorem{proposition}{Proposition}
\usepackage{xcolor}
\usepackage{colortbl}
\usepackage{tcolorbox}
\usepackage{placeins}
\definecolor{calloutblue}{HTML}{2868C7}
\definecolor{clipblue}{HTML}{2F6FA3}
\definecolor{sigorange}{HTML}{D07A35}
\definecolor{pegreen}{HTML}{3B8C6E}
\definecolor{goodgreen}{HTML}{287A52}
\definecolor{badred}{HTML}{B64B52}
\newcommand{\goodcell}[1]{\cellcolor{goodgreen!13}\textcolor{goodgreen}{\bfseries #1}}
\newcommand{\badcell}[1]{\cellcolor{badred!13}\textcolor{badred}{\bfseries #1}}
\newcommand{\neutralcell}[1]{\cellcolor{black!5}\textcolor{black!65}{\bfseries #1}}
\newcommand{\goodnum}[1]{\textcolor{goodgreen}{\bfseries #1}}
\newcommand{\badnum}[1]{\textcolor{badred}{\bfseries #1}}
\newcommand{\neutralnum}[1]{\textcolor{black!65}{\bfseries #1}}
\newtcolorbox{insightbox}{
  colback=calloutblue!4,
  colframe=calloutblue,
  boxrule=.55pt,
  arc=1.5mm,
  outer arc=1.5mm,
  left=3pt,
  right=3pt,
  top=2pt,
  bottom=2pt,
  before skip=3pt,
  after skip=3pt,
  before upper=\raggedright,
  fontupper=\small
}
\newcommand{\fpr}{FPR95}
\newcommand{\ceg}{\textsc{CEG}}
\newcommand{\cegp}{\textsc{CEG-P}}

\title{Level, Sharpness, and Corpus: \\ Why Zero-Shot OOD Detector Rankings Do Not Transfer}

\author{
    Ignacio M. De la Jara\textsuperscript{1, 3},
    Cristian Rodriguez-Opazo\textsuperscript{2},
    Stephen Gould\textsuperscript{2},
    Damith Ranasinghe\textsuperscript{1,3}
}

\affiliations{
    \textsuperscript{1}University of Adelaide,
    \textsuperscript{2}Australian National University,
    \textsuperscript{3}Naval Group Pacific\\
    ignacio.mezadelajara@adelaide.edu.au
}

\begin{document}
\maketitle

\begin{abstract}
Selecting a zero-shot out-of-distribution (OOD) detector for a new deployment
is typically based on benchmark rankings, implicitly assuming that the
highest-ranked detector will transfer across domains. We show that this
assumption does not hold. Through a controlled portability audit across
seventeen in-distribution datasets, three vision-language models, and seven
representative zero-shot OOD detectors, we find that detector rankings reverse
across deployments, every detector exceeds $80\%$ FPR95 on at least one domain,
and the preferred detector depends on both the in-distribution data and the
underlying VLM. We trace these reversals to complementary evidence channels in
vision-language logits. Corpus-free detectors rely on different combinations of
absolute match level and relative or spatial sharpness, while WordNet-based
methods additionally depend on external semantic coverage. A simple proposition
shows that level and sharpness cannot generally be recovered from one another,
explaining why no single detector transfers reliably across deployments.
Motivated by this diagnosis, we introduce the Complementary Evidence Guard
(CEG), a detector-agnostic wrapper that preserves complementary evidence through
a non-compensatory fusion of the base detector, level, and sharpness using only
empirical in-distribution percentiles. Controls replacing these channels with
entropy, logit variance, or random noise do not reproduce the gains. Without
OOD samples, auxiliary corpora, or learned fusion, CEG reduces detector
sensitivity and improves GL-MCM from $38.1$ to $28.8$ and MCM from $42.6$ to
$30.5$ family-balanced FPR95.
\end{abstract}

\begin{figure}[!t]
\centering
\includegraphics[width=\columnwidth]{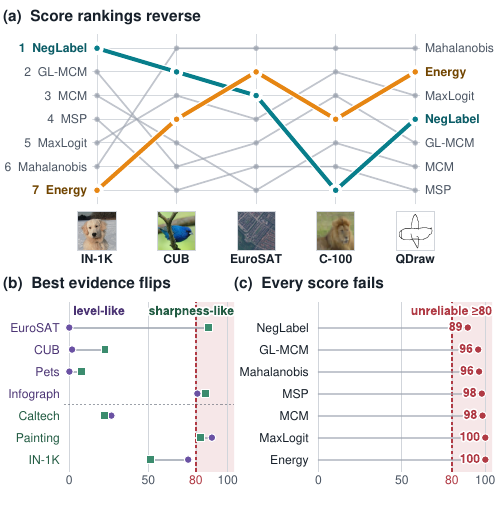}
\vspace{-2em}
\caption{\textbf{Benchmark winners do not transfer across ID domains.}
(a) Detector rankings reverse across ID domains.
(b) Different ID domains favour different evidence.
(c) Every detector fails on at least one ID domain, showing that no single
zero-shot score is universally reliable.}
\label{fig:teaser}
\vspace{-1em}
\end{figure}

\section{Introduction}

Frozen vision-language models (VLMs) have made zero-shot
out-of-distribution (OOD) detection practical across a wide
range of applications. Given an image and a set of
in-distribution (ID) class names, a detector can estimate
whether the image belongs to the ID distribution without
task-specific retraining
\cite{radford2021clip,ming2022mcm}. The remaining challenge
is deciding which detector should be deployed. Existing
methods are typically ranked on a benchmark, and the
highest-performing score is then reused in new deployment
settings. Figure~\ref{fig:teaser} shows that this practice is
unreliable. Across a seventeen-dataset CLIP audit,
detector rankings reverse repeatedly as the ID domain
changes. The same conclusion holds across three
vision-language models in a separate cross-VLM
portability study. The detectors themselves remain valid.
What fails is the assumption that benchmark rankings
transfer across deployments.

This assumption has become widespread because benchmark
evaluation has driven much of the recent progress in
zero-shot OOD detection. Methods are commonly compared on
ImageNet-centred benchmarks, where improvements are often
reported as a single average across ID--OOD shifts.
Consequently, the detector with the best benchmark
performance is frequently treated as the preferred
deployment choice. Previous work questioned whether these
benchmarks faithfully reflect deployment because of dataset
contamination and benchmark bias
\cite{bitterwolf2023ninco,fort2021limits}. We ask a
different question. Even if the benchmark itself is valid,
does the detector ranking obtained on one ID domain remain
valid when the ID domain changes?

To answer this question, we conduct a controlled portability
study while holding implementations, preprocessing, and
evaluation protocols fixed. Across seventeen ID domains,
three VLMs, and seven representative post-hoc OOD
detectors, rankings reverse repeatedly as either the ID
domain or the underlying VLM changes. Every detector
exceeds $80\%$ FPR95 on at least one ID domain, and no
method remains consistently superior. The problem is
therefore not simply identifying the detector with the best
average benchmark performance, but understanding why
different detectors succeed on different ID domains.

Figure~\ref{fig:teaser} suggests the underlying mechanism.
Existing detectors rely on complementary evidence channels.
\emph{Level} measures how strongly an image matches the ID
classes in absolute terms, whereas \emph{sharpness} measures
how clearly one class stands out relative to the remaining
vocabulary and spatial locations. Some methods additionally
exploit external semantic coverage. Different ID domains
preserve different combinations of these signals, making no
single evidence channel universally sufficient. This
perspective explains why detector rankings reverse across
ID domains and suggests preserving complementary evidence
rather than relying on a single score. Instead of seeking a
universally best detector, the goal becomes reducing
deployment sensitivity by retaining evidence that
individual scores discard, leading to more stable decisions
across ID domains.

This perspective changes the deployment objective. If
different ID domains require different evidence, selecting
a single benchmark winner is no longer sufficient. Instead,
deployment should preserve complementary evidence so that
decisions remain reliable when the informative evidence
changes across ID domains. Guided by this principle, we
propose the \emph{Complementary Evidence Guard} (CEG), a
detector-agnostic framework that calibrates complementary
evidence using unlabeled ID data and combines it through a
non-compensatory decision rule. Rather than replacing
existing detectors, CEG is designed to make them more
robust to changes in the ID domain by preserving the
evidence that individual scores discard.

\paragraph{Contributions.}
\begin{itemize}

\item \textbf{A portability audit.}
We show that zero-shot OOD detector rankings do not transfer across ID domains or VLM families.

\item \textbf{An evidence-level explanation.}
We explain ranking reversals through complementary evidence channels whose utility depends on the ID domain.

\item \textbf{A complementary evidence guard.}
We introduce a detector-agnostic wrapper that preserves complementary evidence without retraining.

\item \textbf{A deployment-oriented evaluation protocol.}
We advocate reporting worst-domain risk and revalidating detectors whenever the ID domain or VLM changes.

\end{itemize}

\begin{insightbox}
\textcolor{calloutblue}{\textbf{In short.}}
Benchmark winners are not necessarily deployment winners.
Different ID domains preserve different evidence, making
complementary evidence more reliable than relying on a
single detector.
\end{insightbox}
\section{Related Work}
\label{sec:related}

\paragraph{Post-hoc OOD detection.} Classical detectors threshold classifier
confidence: maximum softmax probability \citep{hendrycks2017baseline}, maximum
logit \citep{hendrycks2022scaling}, free energy \citep{liu2020energy}, or
Mahalanobis distance \citep{lee2018mahalanobis}. VLM variants apply related
statistics to image-text logits: MCM uses the ID-vocabulary softmax
\citep{ming2022mcm}; GL-MCM adds a local term \citep{miyai2025glmcm}; and
NegLabel adds WordNet negatives \citep{jiang2024neglabel,miller1995wordnet}.
ZOC, CLIPN, NegPrompt, and LoCoOp instead add generated negatives, a learned
``no'' embedding, negative prompts, or ID prompt adaptation
\citep{esmaeilpour2022zoc,wang2023clipn,li2024negprompt,miyai2023locoop}.
LogitGap contrasts selected logits \citep{liang2025logitgap}. Conjugated
Semantic Pool (CSP) mines external WordNet noun/adjective prompts and scores
positive mass against that expanded pool \citep{chen2024csp}. Later variants robustly weight negatives
\citep{peng2026closelook} or adapt proxies at test time
\citep{zhang2024adaneg}. Large-scale classifier studies also expose the limits
of single-benchmark conclusions \citep{fort2021limits}. Most VLM evidence
remains ImageNet-centered; we vary the
ID dataset and isolate the external corpus's role.

\paragraph{Resource assumptions.} MSP, MaxLogit, Energy, and MCM need only the
ID vocabulary; GL-MCM also reads patches \citep{ming2022mcm,miyai2025glmcm};
Mahalanobis fits ID statistics \citep{lee2018mahalanobis}; and NegLabel, CSP,
and AdaNeg add external semantics
\citep{jiang2024neglabel,chen2024csp,zhang2024adaneg}. Prompt learning adds
adaptation. We keep these resource tiers separate after wrapping; \ceg{}
calibrates its channels to ID percentiles before taking their minimum.

\paragraph{Benchmarks and VLM portability.} OpenOOD standardizes ImageNet-scale
evaluation \citep{yang2022openood,zhang2023openood}. Prior work varies CLIP
encoders, and NegLabel/EOE also test ALIGN, GroupViT, or AltCLIP, but keep
ImageNet-1K and its usual OOD suite
\citep{ming2022mcm,jiang2024neglabel,cao2024eoe}. We instead test score-rank
invariance while changing both VLM and ID domain.

\begin{figure*}[t]
\centering
\includegraphics[width=0.99\textwidth]{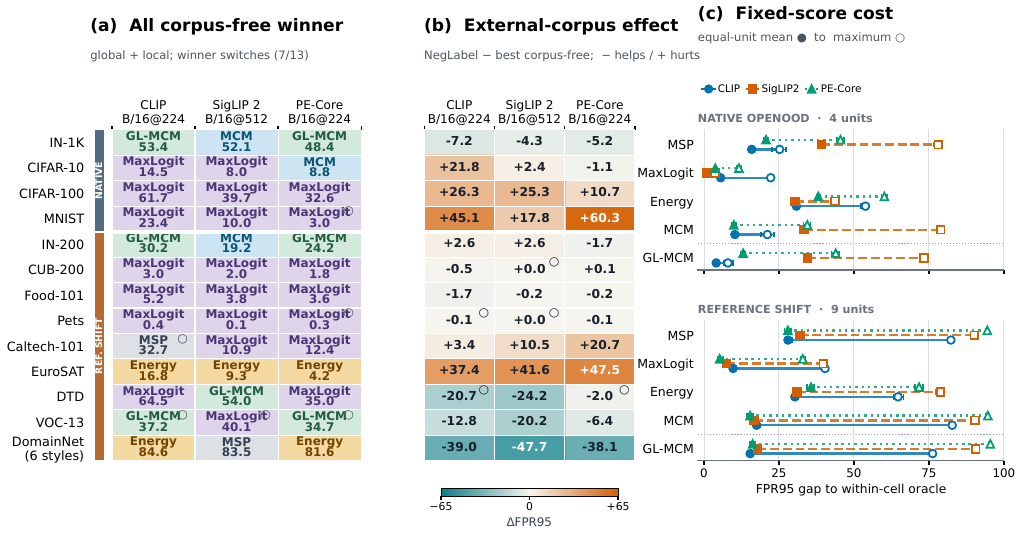}
\vspace{-1em}
\caption{\textbf{Portability fails because detectors preserve different evidence.}
(a) The best corpus-free detector changes across ID domains.
(b) The contribution of external semantic coverage depends on the ID domain and VLM.
(c) The cost of fixing one score: filled markers show the average FPR95 gap to the best score for each unit, and open markers show the worst-case gap. Farther right means worse transfer.}
\label{fig:multibackbone}
\end{figure*}

\section{The Portability Audit}
\label{sec:audit}

Zero-shot OOD work primarily develops and evaluates methods
on CLIP and ImageNet, the dominant benchmark in the
vision-language OOD literature, before deploying the same
formula on different wID domains or VLMs. This implicitly
assumes that a runnable detector retains its empirical
rank. We begin by testing that assumption.

\paragraph{A benchmark winner is not a deployment winner.}
If benchmark rankings were portable, the detector with the
best benchmark average should also be the safest deployment
choice. Figure~\ref{fig:multibackbone}c shows that this is
not the case. Every fixed score exhibits a large gap between
its mean and worst-domain performance. Across the complete
seventeen-domain CLIP study (Appendix
Table~\ref{tab:audit}), every detector exceeds $80$ \fpr{}
on at least one deployment, and the best worst case is
still $89$. Selecting the lowest benchmark mean therefore
does not control the failures practitioners encounter after
the ID domain changes.

To distinguish failures caused by the ID domain
from those introduced by the benchmark itself, we evaluate
two complementary deployment settings. \emph{Native}
evaluations use each dataset's official OpenOOD-v1.5
benchmark and OOD suite \citep{zhang2023openood},
whereas \emph{reference-shift} evaluations score different
ID datasets against a shared ImageNet-based OOD
reference. Together these separate changes caused by the
deployment domain from those introduced by the benchmark.
Complete implementation details are given in
App.~\ref{app:protocol}.

\paragraph{The model changes the answer too.}
Figure~\ref{fig:multibackbone}a shows that replacing only
the underlying VLM also changes the preferred detector. The
all-five winner changes in $7/13$ deployment units (two
native and five reference-shift), while the frozen global
subset still changes in $5/13$ units after controls. Winner
probability exceeds $0.95$ in $33/39$ all-method cells, so
these reversals are not merely close statistical ties.
Neither the deployment domain nor the representation alone
determines a universally best detector, suggesting that the
answer instead depends on the evidence each detector
preserves.

\paragraph{The failures are complementary.}
Figure~\ref{fig:multibackbone}b explains why rankings
reverse. NegLabel succeeds when WordNet brackets the
deployment shift but fails on CIFAR-like domains, whereas
Mahalanobis benefits from compact ID feature clusters but
deteriorates on broad ImageNet deployments. Their nearly
disjoint strengths indicate that domains do not reward
universally better detectors; they reward different evidence
channels. This naturally raises the central question: 
what evidence do existing detectors actually preserve? 
We answer that question in Section~\ref{sec:mechanism}.
% Section~\ref{sec:mechanism} traces these
% differences to absolute match level and relative/spatial
% sharpness, while the benchmark analysis below shows why
% the evaluation itself must also reflect the intended
% deployment setting.

\begin{figure}[t]
\centering
\includegraphics[width=\columnwidth]{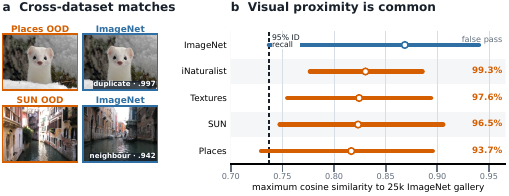}
\caption{\textbf{Benchmark composition is part of the audit.}
(a) A verified Places/ImageNet duplicate and a selected SUN
neighbour illustrate that label-defined OOD need not be
visually disjoint. (b) Maximum-cosine distributions overlap
substantially between ImageNet and the traditional OOD
sources. Full protocol: Appendix~\ref{app:benchmark-choice}.}
\label{fig:classic-overlap-compact}
\end{figure}

\paragraph{Benchmark composition is part of the audit.}
Portability should be evaluated independently of artefacts
introduced by the benchmark itself. Traditional ImageNet
protocols define far-OOD using iNaturalist, SUN, Places,
and Textures \citep{huang2021mos}, although label-defined
OOD need not be visually or data disjoint. Figure
\ref{fig:classic-overlap-compact} shows both a verified
Places/ImageNet duplicate and substantial overlap between
the corresponding maximum-cosine distributions, where
$93.7$--$99.3\%$ of nominally OOD images pass the
ImageNet $95\%$ recall threshold. This probe does not
estimate contamination prevalence, but it shows that the
traditional benchmark is not visually clean. We therefore
adopt OpenOOD-v1.5, which explicitly separates near- and
far-OOD shifts while avoiding these known benchmark issues
\citep{zhang2023openood,bitterwolf2023ninco}.
\section{Mechanism: Level and Sharpness Are Complementary Evidence}
\label{sec:mechanism}

\begin{figure}[t]
\centering
\includegraphics[width=0.99\columnwidth]{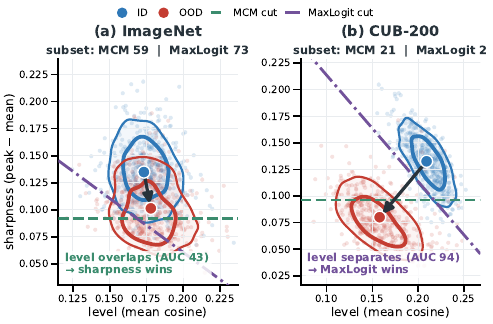}
\caption{\textbf{Level and sharpness separate different shifts.}
Sharpness separates ImageNet, whereas absolute match level separates CUB-200,
explaining the MCM/MaxLogit reversal. Contours and header FPR95 use a diagnostic
subset (20 ID images/class; at most 3,000/OOD source); Tables~\ref{tab:ceg}
and~\ref{tab:audit} report separate protocols. Horizontal position is mean
cosine, the level term in MaxLogit.}
\label{fig:channels}
\end{figure}

The portability audit showed that detector rankings reverse because different
ID domains reward different evidence channels. We now ask what that
evidence is. Figure~\ref{fig:channels} shows that CLIP cosine logits contain
two partially independent signals:

\begin{itemize}
\item the \textbf{level} $\max_c \ell_c$: how strongly the image matches its
best anchor in absolute terms;
\item the \textbf{sharpness}
$\max_c\ell_c-\operatorname{mean}_c\ell_c$: how much the best match stands out
from the remaining vocabulary and, spatially, how much the strongest image
region stands out from the rest of the image.
\end{itemize}

Different samples can fail through different channels. An
image may strongly match the ID vocabulary in absolute
terms yet lack a coherent class peak, whereas another may
produce a sharp winning class despite remaining atypical in
absolute level. A single score therefore cannot capture every
failure, motivating the non-compensatory guard introduced
next.

\paragraph{How existing detectors use these channels.}
The previous analysis explains why both level and sharpness
are needed. Existing detectors, however, do not preserve
both. Instead, each detector implicitly commits to one form
of evidence, which explains why their rankings reverse across ID domains.

\paragraph{MCM discards level.}
The softmax formulation of MCM obscures a simple but
important fact.

\begin{proposition}
\label{prop:shift}
$S_{\mathrm{MCM}}(x)=\max_c \operatorname{softmax}(\ell/T)_c$ is invariant to
additive logit shifts $\ell_c \mapsto \ell_c + \gamma(x)$, and for
$T\!\gg\!R$ (logit range),
$S_{\mathrm{MCM}} = \tfrac{1}{K} + \tfrac{1}{KT}\left(\max_c\ell_c -
\operatorname{mean}_c\ell_c\right) + O(R^2/T^2)$.
\end{proposition}

Proposition~\ref{prop:shift} identifies MCM as a
shift-invariant approximation to the centred maximum,
thereby preserving only the sharpness channel. The
complementary decomposition follows from the
log-sum-exp identity. Writing
$\ell_c=\bar{\ell}+\delta_c$ with
$\sum_c\delta_c=0$,

\[
\operatorname{Energy}(\ell)
=
\log\sum_c e^{\ell_c}
=
\log K+\bar{\ell}
+\frac{\operatorname{Var}(\ell)}{2}
+O(\|\delta\|^3),
\]
thus, Energy preserves the absolute level to leading order,
with only a second-order dependence on logit variance.
Empirically, both approximations are essentially exact in
our setting (Appendix~E.3): MCM differs from the centred
maximum by less than $0.2$ \fpr{} across all core
benchmarks, while Energy closely matches the mean level.
Together, MCM and Energy form two complementary
extremes of the same logit decomposition: one preserves
sharpness, whereas the other preserves level.

This decomposition makes the portability failures
interpretable. ID domains separated primarily by sharpness favour
MCM-like scores, whereas deployments
separated primarily by level favour Energy- or
MaxLogit-like scores. Because the informative component 
changes with the ID domain, neither family can
retain a universal ranking.

This immediately explains the behaviour observed in
Figure~\ref{fig:channels}. On ImageNet, a $1{,}000$-anchor
vocabulary gives ID images one sharp winner while far-OOD
matches nothing sharply. On EuroSAT and Quickdraw,
however, ID images themselves match no anchor sharply
(satellite tiles and line drawings are far from CLIP's
photographic text concepts), so the sharpness channel
contains little ID/OOD information and MCM approaches
chance performance. Conversely, on CUB-200 the
\emph{level} almost perfectly separates ID from OOD, yet
shift-invariant scores discard that signal, giving MCM
$23$ versus MaxLogit $2$.

\paragraph{Vocabulary size confounds the score.} Increasing the number of
anchors $K$ changes both maxima and normalisers, so part of ImageNet's apparent
ranking is the vocabulary rather than the images. In controlled class subsets
(Figure~\ref{fig:klabels}a), MaxLogit pays a growing penalty to MCM. Growing
the label set alone can therefore change detector behaviour.

\begin{figure}[t]
\centering
\includegraphics[width=0.97\columnwidth]{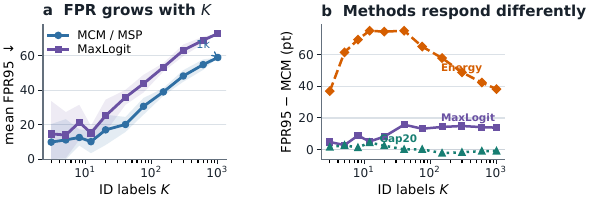}
\caption{\textbf{Label-set construction changes the detector.}
(a) Mean \fpr{} on a fixed suite under ImageNet class subsampling (eight draws;
bands $\pm1$ s.d.). (b) Other global-logit scores minus MCM; positive favours
MCM, and MSP${=}$MCM under this scoring convention.}
\label{fig:klabels}
\end{figure}

\paragraph{Each detector commits to different evidence.}
The previous analysis explains the behaviour of the
corpus-free detectors. MSP and MCM preserve sharpness
because softmax normalisation discards the level;
MaxLogit, Energy, and Mahalanobis preserve the level (or
its feature-space analogue); and GL-MCM extends
sharpness with local evidence. The portability audit
therefore becomes interpretable: level-separable domains
(CUB, Pets, Food, EuroSAT, Quickdraw, and MNIST)
reward MaxLogit and Mahalanobis, whereas
sharpness-separable domains (ImageNet-1K/200 and
Caltech) reward MCM and GL-MCM. Domains where
neither channel is reliable (Painting and Infograph) leave
every corpus-free detector at or above $49$ \fpr{}.

This leaves one remaining family of detectors. NegLabel
and CSP do not rely solely on the ID vocabulary. Instead,
they augment it with a WordNet-derived semantic corpus,
introducing a third source of evidence that depends on the
coverage of that external vocabulary rather than only on the
image representations themselves.

\begin{figure}[t]
\centering
\includegraphics[width=\columnwidth]{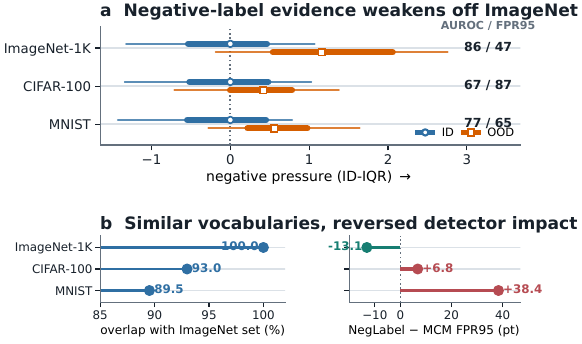}
\caption{\textbf{External semantic evidence depends on corpus coverage.}
(a) Negative-label evidence separates ImageNet but weakens
off ImageNet. (b) NegMining selects nearly the same
WordNet vocabulary across deployments, yet detector
performance reverses. Corpus coverage, not vocabulary
selection, limits transfer.}
\label{fig:neglabel}
\end{figure}

\subsection{External Semantic Evidence Depends on Corpus Coverage}
\label{sec:neg}

Unlike the corpus-free detectors, NegLabel and CSP augment
the ID vocabulary with WordNet-derived concepts. Their
success therefore depends not only on the image evidence,
but also on whether this external semantic corpus brackets
the ID–OOD shift. Figure~\ref{fig:neglabel} asks two
related questions: \textit{does the mined vocabulary provide useful
OOD evidence, and does NegMining adapt that vocabulary
when the ID domain changes?}

Panel~(a) shows that the external semantic evidence
separates ID and OOD on ImageNet but rapidly loses
discriminative power on CIFAR-100 and MNIST. Panel~(b)
explains this reversal. Although theID domain changes substantially, 
NegMining selects almost the same
WordNet vocabulary, with $89.5$ to $93.0\%$ overlap with
the ImageNet negatives. Nevertheless, NegLabel changes
from improving MCM by $13.1$ FPR95 points on ImageNet
to degrading it by $6.8$ and $38.4$ points on CIFAR-100
and MNIST, respectively. The limitation therefore lies not
in the mining procedure, but in the coverage of the external
semantic corpus. WordNet brackets ImageNet-style shifts
well, yet provides much weaker evidence for other
deployment domains. CSP relies on the same WordNet
corpus and therefore inherits the same dependence on
corpus coverage.

% \begin{insightbox}
% \textcolor{calloutblue}{\textbf{Mechanism takeaway.}}
% Sharpness, absolute level, and external-corpus coverage fail on different
% shifts. Normalising away one cue can remove exactly the evidence a domain needs.
% \end{insightbox}

\begin{insightbox}
\textcolor{calloutblue}{}
The audit identifies three complementary evidence
channels. Because deployment domains preserve different
channels, no fixed detector transfers reliably. This
diagnosis naturally motivates CEG, which preserves
complementary evidence rather than relying on a single
score.
\end{insightbox}

% AUTO-GENERATED by scripts/104_make_ceg_paper_tables.py
\begin{table*}[t]
\centering\scriptsize
\setlength{\tabcolsep}{3.0pt}
\renewcommand{\arraystretch}{0.86}
\caption{\textbf{CEG: strict results and full-domain coverage.} (a) Raw$\to$\ceg{} FPR95$\downarrow$/AUROC$\uparrow$ on strict tasks. (b) Every entry is $\Delta$FPR/$\Delta$AUC (\ceg{}$-$Raw): negative $\Delta$FPR and positive $\Delta$AUC denote improvement. Colors are metric-specific. In panel (a), each subcolumn is keyed independently and only regressions are tinted; in panel (b), the two numbers are colored independently. Green/red/gray denotes improvement/regression/$|\Delta|\le0.10$ tie. DomainNet counts once in the 13-unit mean. Aggregation sensitivity is in Appendix Table~\ref{tab:ceg-weighting-sensitivity}.}
\label{tab:ceg}
\textbf{(a) Strict five-task headline: Raw$\to$CEG}\par\vspace{1pt}
\setlength{\tabcolsep}{1.5pt}
\begin{tabular*}{\textwidth}{@{\extracolsep{\fill}}l*{12}{r}}
\toprule
& \multicolumn{2}{c}{IN-1K} & \multicolumn{2}{c}{IN-200} & \multicolumn{2}{c}{C-10} & \multicolumn{2}{c}{C-100} & \multicolumn{2}{c}{MNIST} & \multicolumn{2}{c}{Family}\\
\cmidrule(lr){2-3}\cmidrule(lr){4-5}\cmidrule(lr){6-7}\cmidrule(lr){8-9}\cmidrule(lr){10-11}\cmidrule(lr){12-13}
Base detector & FPR$\downarrow$ & AUC$\uparrow$ & FPR$\downarrow$ & AUC$\uparrow$ & FPR$\downarrow$ & AUC$\uparrow$ & FPR$\downarrow$ & AUC$\uparrow$ & FPR$\downarrow$ & AUC$\uparrow$ & FPR$\downarrow$ & AUC$\uparrow$\\
\midrule
\rowcolor{clipblue!7}\multicolumn{13}{l}{\textbf{Class-name scores; no external corpus}}\\
MaxLogit & \goodnum{72.7$\!\to\!$54.3} & \goodnum{77.3$\!\to\!$82.3} & \goodnum{45.7$\!\to\!$30.2} & \goodnum{89.3$\!\to\!$92.2} & \goodnum{14.5$\!\to\!$11.7} & \goodnum{96.2$\!\to\!$97.2} & \goodnum{61.5$\!\to\!$52.6} & \goodnum{85.7$\!\to\!$87.5} & \goodnum{23.5$\!\to\!$15.8} & \goodnum{95.1$\!\to\!$97.1} & \goodnum{40.2$\!\to\!$30.1} & \goodnum{89.8$\!\to\!$92.2}\\
\addlinespace[0.35pt]
Energy ($\tau{=}1$) & \goodnum{97.0$\!\to\!$56.9} & \goodnum{42.7$\!\to\!$79.3} & \goodnum{95.6$\!\to\!$33.0} & \goodnum{46.6$\!\to\!$90.8} & \goodnum{68.8$\!\to\!$13.4} & \goodnum{70.5$\!\to\!$96.8} & \goodnum{77.9$\!\to\!$57.2} & \goodnum{64.4$\!\to\!$86.6} & \goodnum{34.4$\!\to\!$16.0} & \goodnum{90.3$\!\to\!$97.0} & \goodnum{68.0$\!\to\!$32.1} & \goodnum{67.5$\!\to\!$91.3}\\
\addlinespace[0.35pt]
MCM & \goodnum{60.0$\!\to\!$53.9} & \goodnum{81.9$\!\to\!$82.9} & \goodnum{32.2$\!\to\!$28.9} & \goodnum{91.7$\!\to\!$92.5} & \goodnum{19.4$\!\to\!$11.5} & \goodnum{94.8$\!\to\!$97.2} & \goodnum{82.4$\!\to\!$56.8} & \goodnum{76.6$\!\to\!$85.9} & \goodnum{30.8$\!\to\!$16.0} & \goodnum{93.5$\!\to\!$97.0} & \goodnum{42.6$\!\to\!$30.5} & \goodnum{88.7$\!\to\!$92.1}\\
\addlinespace[0.35pt]
GL-MCM$^\dagger$ & \goodnum{53.1$\!\to\!$52.7} & \badcell{83.2$\!\to\!$82.9} & \goodnum{28.7$\!\to\!$27.3} & \goodnum{92.0$\!\to\!$92.5} & \goodnum{16.4$\!\to\!$11.1} & \goodnum{95.4$\!\to\!$97.1} & \goodnum{67.8$\!\to\!$53.6} & \goodnum{80.0$\!\to\!$86.4} & \goodnum{31.2$\!\to\!$14.0} & \goodnum{93.0$\!\to\!$97.4} & \goodnum{38.1$\!\to\!$28.8} & \goodnum{89.4$\!\to\!$92.3}\\
\addlinespace[0.35pt]
20\%-nonmax gap (abl.) & \goodnum{59.1$\!\to\!$53.2} & \goodnum{82.7$\!\to\!$83.2} & \goodnum{29.2$\!\to\!$27.3} & \goodnum{92.4$\!\to\!$92.8} & \goodnum{34.6$\!\to\!$12.3} & \goodnum{92.8$\!\to\!$97.0} & \goodnum{88.0$\!\to\!$57.9} & \goodnum{76.3$\!\to\!$85.9} & \goodnum{76.7$\!\to\!$19.9} & \goodnum{78.8$\!\to\!$95.8} & \goodnum{60.7$\!\to\!$31.8} & \goodnum{83.7$\!\to\!$91.7}\\
\addlinespace[0.35pt]
LogitGap & \goodnum{59.1$\!\to\!$53.2} & \goodnum{82.7$\!\to\!$83.2} & \goodnum{29.2$\!\to\!$27.2} & \goodnum{92.4$\!\to\!$92.8} & \goodnum{26.5$\!\to\!$11.8} & \goodnum{93.9$\!\to\!$97.1} & \goodnum{88.0$\!\to\!$57.9} & \goodnum{76.3$\!\to\!$85.9} & \goodnum{50.7$\!\to\!$18.9} & \goodnum{87.5$\!\to\!$96.2} & \goodnum{50.7$\!\to\!$31.3} & \goodnum{86.7$\!\to\!$91.9}\\
\addlinespace[0.35pt]
\midrule
\rowcolor{pegreen!7}\multicolumn{13}{l}{\textbf{Unlabeled-ID feature fit}}\\
Mahalanobis & \goodnum{96.6$\!\to\!$57.3} & \goodnum{62.1$\!\to\!$80.5} & \goodnum{86.9$\!\to\!$32.9} & \goodnum{82.0$\!\to\!$91.8} & \goodnum{18.1$\!\to\!$8.2} & \goodnum{95.9$\!\to\!$98.0} & \goodnum{37.7$\!\to\!$31.3} & \goodnum{90.9$\!\to\!$93.4} & \badcell{0.2$\!\to\!$0.4} & \neutralnum{99.9$\!\to\!$99.8} & \goodnum{40.0$\!\to\!$21.7} & \goodnum{88.5$\!\to\!$93.9}\\
\addlinespace[0.35pt]
\midrule
\rowcolor{sigorange!8}\multicolumn{13}{l}{\textbf{External WordNet corpus/mining}}\\
NegLabel & \goodnum{46.2$\!\to\!$46.0} & \badcell{85.9$\!\to\!$85.0} & \goodnum{29.3$\!\to\!$26.1} & \goodnum{93.0$\!\to\!$93.2} & \goodnum{36.1$\!\to\!$11.9} & \goodnum{91.9$\!\to\!$97.1} & \goodnum{88.7$\!\to\!$59.4} & \goodnum{66.7$\!\to\!$85.0} & \goodnum{69.0$\!\to\!$15.8} & \goodnum{76.4$\!\to\!$96.9} & \goodnum{56.4$\!\to\!$29.2} & \goodnum{81.7$\!\to\!$92.3}\\
\addlinespace[0.35pt]
CSP & \badcell{41.5$\!\to\!$42.3} & \badcell{87.4$\!\to\!$86.2} & \goodnum{25.2$\!\to\!$24.7} & \badcell{93.9$\!\to\!$93.7} & \goodnum{30.2$\!\to\!$11.7} & \goodnum{94.2$\!\to\!$97.1} & \goodnum{81.1$\!\to\!$54.9} & \goodnum{74.7$\!\to\!$86.0} & \goodnum{56.7$\!\to\!$15.4} & \goodnum{84.6$\!\to\!$97.0} & \goodnum{48.6$\!\to\!$27.4} & \goodnum{86.6$\!\to\!$92.8}\\
\addlinespace[0.35pt]
\bottomrule
\end{tabular*}
\vspace{2pt}\par\textbf{(b) Every pitfall-audit unit: $\Delta$FPR/$\Delta$AUC (CLIP-B/16)}\par
\tiny\setlength{\tabcolsep}{0.7pt}\renewcommand{\arraystretch}{0.92}
\begin{tabular*}{\textwidth}{@{\extracolsep{\fill}}lrrrrrrrrrrrrrr}
\toprule
Base & IN-1K & C-10 & C-100 & MNIST & IN-200 & CUB & Food & Pets & Cal & Euro & DTD & VOC & DN & 13-unit\\
\midrule
MSP & \textcolor{goodgreen}{\bfseries -12.4}\hspace{0.15pt}/\hspace{0.15pt}\textcolor{goodgreen}{\bfseries +4.2} & \textcolor{goodgreen}{\bfseries -14.3}\hspace{0.15pt}/\hspace{0.15pt}\textcolor{goodgreen}{\bfseries +3.3} & \textcolor{goodgreen}{\bfseries -29.2}\hspace{0.15pt}/\hspace{0.15pt}\textcolor{goodgreen}{\bfseries +10.6} & \textcolor{goodgreen}{\bfseries -21.1}\hspace{0.15pt}/\hspace{0.15pt}\textcolor{goodgreen}{\bfseries +6.5} & \textcolor{goodgreen}{\bfseries -7.7}\hspace{0.15pt}/\hspace{0.15pt}\textcolor{goodgreen}{\bfseries +1.9} & \textcolor{goodgreen}{\bfseries -74.7}\hspace{0.15pt}/\hspace{0.15pt}\textcolor{goodgreen}{\bfseries +30.1} & \textcolor{goodgreen}{\bfseries -15.8}\hspace{0.15pt}/\hspace{0.15pt}\textcolor{goodgreen}{\bfseries +3.2} & \textcolor{goodgreen}{\bfseries -37.2}\hspace{0.15pt}/\hspace{0.15pt}\textcolor{goodgreen}{\bfseries +8.4} & \textcolor{goodgreen}{\bfseries -5.4}\hspace{0.15pt}/\hspace{0.15pt}\textcolor{goodgreen}{\bfseries +2.1} & \textcolor{goodgreen}{\bfseries -34.1}\hspace{0.15pt}/\hspace{0.15pt}\textcolor{goodgreen}{\bfseries +36.3} & \textcolor{goodgreen}{\bfseries -25.1}\hspace{0.15pt}/\hspace{0.15pt}\textcolor{goodgreen}{\bfseries +16.1} & \textcolor{goodgreen}{\bfseries -6.4}\hspace{0.15pt}/\hspace{0.15pt}\textcolor{goodgreen}{\bfseries +4.7} & \textcolor{goodgreen}{\bfseries -5.0}\hspace{0.15pt}/\hspace{0.15pt}\textcolor{goodgreen}{\bfseries +5.3} & \textcolor{goodgreen}{\bfseries -22.2}\hspace{0.15pt}/\hspace{0.15pt}\textcolor{goodgreen}{\bfseries +10.2}\\
\addlinespace[0.35pt]
MaxLogit & \textcolor{goodgreen}{\bfseries -21.1}\hspace{0.15pt}/\hspace{0.15pt}\textcolor{goodgreen}{\bfseries +5.4} & \textcolor{goodgreen}{\bfseries -3.4}\hspace{0.15pt}/\hspace{0.15pt}\textcolor{goodgreen}{\bfseries +0.9} & \textcolor{goodgreen}{\bfseries -8.9}\hspace{0.15pt}/\hspace{0.15pt}\textcolor{goodgreen}{\bfseries +1.8} & \textcolor{goodgreen}{\bfseries -8.0}\hspace{0.15pt}/\hspace{0.15pt}\textcolor{goodgreen}{\bfseries +1.9} & \textcolor{goodgreen}{\bfseries -14.4}\hspace{0.15pt}/\hspace{0.15pt}\textcolor{goodgreen}{\bfseries +2.5} & \textcolor{badred}{\bfseries +0.3}\hspace{0.15pt}/\hspace{0.15pt}\textcolor{badred}{\bfseries -0.1} & \textcolor{goodgreen}{\bfseries -0.5}\hspace{0.15pt}/\hspace{0.15pt}\textcolor{black!60}{\bfseries -0.1} & \textcolor{black!60}{\bfseries 0.0}\hspace{0.15pt}/\hspace{0.15pt}\textcolor{badred}{\bfseries -0.1} & \textcolor{goodgreen}{\bfseries -16.6}\hspace{0.15pt}/\hspace{0.15pt}\textcolor{goodgreen}{\bfseries +4.4} & \textcolor{badred}{\bfseries +0.6}\hspace{0.15pt}/\hspace{0.15pt}\textcolor{goodgreen}{\bfseries +1.4} & \textcolor{goodgreen}{\bfseries -6.4}\hspace{0.15pt}/\hspace{0.15pt}\textcolor{goodgreen}{\bfseries +2.7} & \textcolor{goodgreen}{\bfseries -15.3}\hspace{0.15pt}/\hspace{0.15pt}\textcolor{goodgreen}{\bfseries +5.0} & \textcolor{goodgreen}{\bfseries -3.6}\hspace{0.15pt}/\hspace{0.15pt}\textcolor{goodgreen}{\bfseries +9.3} & \textcolor{goodgreen}{\bfseries -7.5}\hspace{0.15pt}/\hspace{0.15pt}\textcolor{goodgreen}{\bfseries +2.7}\\
\addlinespace[0.35pt]
Energy & \textcolor{goodgreen}{\bfseries -39.1}\hspace{0.15pt}/\hspace{0.15pt}\textcolor{goodgreen}{\bfseries +38.2} & \textcolor{goodgreen}{\bfseries -55.5}\hspace{0.15pt}/\hspace{0.15pt}\textcolor{goodgreen}{\bfseries +26.3} & \textcolor{goodgreen}{\bfseries -21.1}\hspace{0.15pt}/\hspace{0.15pt}\textcolor{goodgreen}{\bfseries +22.4} & \textcolor{goodgreen}{\bfseries -18.5}\hspace{0.15pt}/\hspace{0.15pt}\textcolor{goodgreen}{\bfseries +6.7} & \textcolor{goodgreen}{\bfseries -60.6}\hspace{0.15pt}/\hspace{0.15pt}\textcolor{goodgreen}{\bfseries +40.5} & \textcolor{goodgreen}{\bfseries -15.7}\hspace{0.15pt}/\hspace{0.15pt}\textcolor{goodgreen}{\bfseries +4.5} & \textcolor{goodgreen}{\bfseries -47.1}\hspace{0.15pt}/\hspace{0.15pt}\textcolor{goodgreen}{\bfseries +11.9} & \textcolor{goodgreen}{\bfseries -17.7}\hspace{0.15pt}/\hspace{0.15pt}\textcolor{goodgreen}{\bfseries +4.8} & \textcolor{goodgreen}{\bfseries -61.5}\hspace{0.15pt}/\hspace{0.15pt}\textcolor{goodgreen}{\bfseries +42.1} & \textcolor{badred}{\bfseries +7.3}\hspace{0.15pt}/\hspace{0.15pt}\textcolor{badred}{\bfseries -2.5} & \textcolor{goodgreen}{\bfseries -16.7}\hspace{0.15pt}/\hspace{0.15pt}\textcolor{goodgreen}{\bfseries +8.8} & \textcolor{goodgreen}{\bfseries -47.4}\hspace{0.15pt}/\hspace{0.15pt}\textcolor{goodgreen}{\bfseries +33.4} & \textcolor{goodgreen}{\bfseries -5.0}\hspace{0.15pt}/\hspace{0.15pt}\textcolor{goodgreen}{\bfseries +25.1} & \textcolor{goodgreen}{\bfseries -30.7}\hspace{0.15pt}/\hspace{0.15pt}\textcolor{goodgreen}{\bfseries +20.2}\\
\addlinespace[0.35pt]
MCM & \textcolor{goodgreen}{\bfseries -6.9}\hspace{0.15pt}/\hspace{0.15pt}\textcolor{goodgreen}{\bfseries +1.0} & \textcolor{goodgreen}{\bfseries -8.3}\hspace{0.15pt}/\hspace{0.15pt}\textcolor{goodgreen}{\bfseries +2.4} & \textcolor{goodgreen}{\bfseries -25.4}\hspace{0.15pt}/\hspace{0.15pt}\textcolor{goodgreen}{\bfseries +9.5} & \textcolor{goodgreen}{\bfseries -16.2}\hspace{0.15pt}/\hspace{0.15pt}\textcolor{goodgreen}{\bfseries +3.6} & \textcolor{goodgreen}{\bfseries -5.3}\hspace{0.15pt}/\hspace{0.15pt}\textcolor{goodgreen}{\bfseries +1.2} & \textcolor{goodgreen}{\bfseries -22.1}\hspace{0.15pt}/\hspace{0.15pt}\textcolor{goodgreen}{\bfseries +4.3} & \textcolor{goodgreen}{\bfseries -3.7}\hspace{0.15pt}/\hspace{0.15pt}\textcolor{goodgreen}{\bfseries +0.7} & \textcolor{goodgreen}{\bfseries -7.0}\hspace{0.15pt}/\hspace{0.15pt}\textcolor{goodgreen}{\bfseries +1.4} & \textcolor{goodgreen}{\bfseries -9.5}\hspace{0.15pt}/\hspace{0.15pt}\textcolor{goodgreen}{\bfseries +2.1} & \textcolor{goodgreen}{\bfseries -34.6}\hspace{0.15pt}/\hspace{0.15pt}\textcolor{goodgreen}{\bfseries +47.1} & \textcolor{goodgreen}{\bfseries -25.7}\hspace{0.15pt}/\hspace{0.15pt}\textcolor{goodgreen}{\bfseries +13.9} & \textcolor{goodgreen}{\bfseries -3.3}\hspace{0.15pt}/\hspace{0.15pt}\textcolor{goodgreen}{\bfseries +3.1} & \textcolor{goodgreen}{\bfseries -5.1}\hspace{0.15pt}/\hspace{0.15pt}\textcolor{goodgreen}{\bfseries +4.8} & \textcolor{goodgreen}{\bfseries -13.3}\hspace{0.15pt}/\hspace{0.15pt}\textcolor{goodgreen}{\bfseries +7.3}\\
\addlinespace[0.35pt]
GL-MCM$^\dagger$ & \textcolor{goodgreen}{\bfseries -0.6}\hspace{0.15pt}/\hspace{0.15pt}\textcolor{badred}{\bfseries -0.3} & \textcolor{goodgreen}{\bfseries -5.5}\hspace{0.15pt}/\hspace{0.15pt}\textcolor{goodgreen}{\bfseries +1.7} & \textcolor{goodgreen}{\bfseries -14.5}\hspace{0.15pt}/\hspace{0.15pt}\textcolor{goodgreen}{\bfseries +6.6} & \textcolor{goodgreen}{\bfseries -17.5}\hspace{0.15pt}/\hspace{0.15pt}\textcolor{goodgreen}{\bfseries +4.4} & \textcolor{goodgreen}{\bfseries -1.7}\hspace{0.15pt}/\hspace{0.15pt}\textcolor{goodgreen}{\bfseries +0.6} & \textcolor{goodgreen}{\bfseries -45.0}\hspace{0.15pt}/\hspace{0.15pt}\textcolor{goodgreen}{\bfseries +12.7} & \textcolor{goodgreen}{\bfseries -0.8}\hspace{0.15pt}/\hspace{0.15pt}\textcolor{black!60}{\bfseries 0.0} & \textcolor{goodgreen}{\bfseries -10.3}\hspace{0.15pt}/\hspace{0.15pt}\textcolor{goodgreen}{\bfseries +2.0} & \textcolor{goodgreen}{\bfseries -6.3}\hspace{0.15pt}/\hspace{0.15pt}\textcolor{goodgreen}{\bfseries +1.1} & \textcolor{goodgreen}{\bfseries -28.9}\hspace{0.15pt}/\hspace{0.15pt}\textcolor{goodgreen}{\bfseries +18.7} & \textcolor{goodgreen}{\bfseries -11.7}\hspace{0.15pt}/\hspace{0.15pt}\textcolor{goodgreen}{\bfseries +5.1} & \textcolor{goodgreen}{\bfseries -1.0}\hspace{0.15pt}/\hspace{0.15pt}\textcolor{goodgreen}{\bfseries +0.7} & \textcolor{goodgreen}{\bfseries -3.2}\hspace{0.15pt}/\hspace{0.15pt}\textcolor{goodgreen}{\bfseries +2.8} & \textcolor{goodgreen}{\bfseries -11.3}\hspace{0.15pt}/\hspace{0.15pt}\textcolor{goodgreen}{\bfseries +4.3}\\
\addlinespace[0.35pt]
\bottomrule
\end{tabular*}
\vspace{1pt}\par\tiny $^\dagger$Native local features. Exact unit values and intervals are in Appendices~\ref{app:hard} and~\ref{app:broad}.
\end{table*}

\section{From Diagnosis to Deployment: The Complementary Evidence Guard}
\label{sec:guard}

The portability audit and its evidence-level explanation
suggest a different deployment strategy. Rather than
selecting the detector that performed best on a benchmark,
deployment should preserve whichever evidence remains
informative after the domain changes. This leads naturally
to the Complementary Evidence Guard (CEG), which
preserves complementary evidence using only unlabeled ID
calibration data.

\paragraph{Complementary channels.}
For a higher-is-ID base detector $B$ and unlabeled ID
calibration set $\mathcal C=\{x_j\}_{j=1}^{n}$, the same VLM
forward pass produces two complementary channels,

\begin{equation}
\begin{aligned}
L(x)
&=
z_{\mathrm{ID}}\!\left(\max_c\ell_c\right)
+
z_{\mathrm{ID}}\!\left(
\frac1k
\sum_{i\in\mathcal T_k(x)}
\max_c\ell_{i,c}
\right),\\
S(x)
&=
z_{\mathrm{ID}}\!\left(
\max_c\ell_c-\operatorname{mean}_c\ell_c
\right)
+
z_{\mathrm{ID}}\!\left(
\max_q(\kappa*s)_q
\right).
\end{aligned}
\label{eq:channels}
\end{equation}

Here $\mathcal T_k(x)$ contains the $k=10$ patches with the
highest relative confidence, so $L$ combines global and
local absolute level, whereas $S$ combines global and
spatial sharpness. Both channels are standardised using ID
statistics estimated from $\mathcal C$.

\paragraph{Complementary Evidence Guard.}

Each channel is mapped to its empirical ID percentile,

\[
U_R(x)=\widehat F_R(R(x)),
\]

where $\widehat F_R$ is the empirical cumulative
distribution function of channel $R$ estimated from
$\mathcal C$. The guard is then

\begin{equation}
G_B(x)
=
\ceg(B)(x)
=
\min_{R\in\{B,L,S\}}
U_R(x).
\label{eq:ceg}
\end{equation}

A sample is accepted whenever $G_B(x)\ge\tau$, where
$\tau$ is fixed using a disjoint ID operating split
(Appendix~\ref{app:protocol}). Taking the minimum directly
implements the deployment principle motivated by the
audit: complementary evidence is non-compensatory, so a
single atypical channel is sufficient to veto acceptance.

The percentile transform is closely related to conformal
ranking
\citep{haroush2022statistical,bates2023conformal,kaur2022idecode}.
Our contribution is not the ranking itself, but identifying
the complementary evidence channels whose preservation
improves deployment robustness. Because the channels are
dependent and share the same calibration data, we make no
finite-sample coverage guarantee.

\paragraph{Evaluation protocol.}
The primary evaluation uses five strict OpenOOD tasks with
fixed unlabeled ID calibration data and held-out test sets.
CEG is developed only on capped versions of these tasks,
whereas the reported results use the full uncapped
benchmarks. Transfer to SigLIP2 and PE-Core freezes all
hyperparameters, and the protected operating point is
evaluated using five-fold ID-only cross-fitting. Neither
configuration uses OOD images or OOD labels during
fitting. Complete protocols are given in
Appendix~\ref{app:protocol}.

\paragraph{Deployment robustness.}

If the diagnosis is correct, preserving both channels should not only
lower error but make the outcome less dependent on which detector was
picked in the first place. That is the claim we test here, and it
holds: \ceg{} improves the family-balanced FPR95 of every wrapped
detector, by between $9.3$ and $35.9$ points, and family AUROC rises
for every base as well (Table~\ref{tab:ceg}), so the effect is not
simply a threshold shift at $95\%$ ID recall. The more interesting
change, however, is not that each detector improves but that they stop
disagreeing. Among the five corpus-free bases, the spread between best
and worst contracts from $29.9$, $22.0$ and $43.3$ FPR95 points down
to $3.3$, $2.3$ and $3.9$, under family-balanced, unweighted $13$-unit
and MNIST-excluded weighting respectively
(Appendix Table~\ref{tab:ceg-weighting-sensitivity}). Which guarded detector comes out
ahead still shifts between these schemes, much as the audit would
predict, but the contraction itself survives all three. 
This is what we mean by reducing sensitivity to the
initial choice: the guard does not crown a winner, it
makes the choice matter less. Table~\ref{tab:ceg-weighting-sensitivity}
therefore confirms the prediction of the portability
audit: the largest corrections occur for detectors that
previously discarded one complementary evidence channel.
Accordingly, the strongest gains are observed for the
one-cue detectors such as Energy and the gap-based
scores, whereas Mahalanobis+CEG reaches the lowest
overall error after its additional Gaussian fit and
GL-MCM+CEG remains the strongest corpus-free endpoint.

A stricter benchmark could still be flattering the guard, so we ask
whether the same interpretation holds across the broader audit, where
61 of the 65 CLIP detector-domain pairs improve. The remaining four
cases are structured rather than scattered: they arise when one noisy
channel vetoes an otherwise strong detector, which is precisely the
trade-off imposed by a non-compensatory rule. The channel ablation
supports the same interpretation (Appendix
Table~\ref{tab:ceg-ablation}): level alone reaches $33.3$ and
sharpness alone $36.4$ family-balanced FPR95, compared with $30.6$
when both are preserved, while removing the local term costs a further
$5.8$ points. Simes and Fisher achieve marginally lower error on these
development tasks, but we retain the minimum because it is the only
rule that directly implements the non-compensatory deployment
principle motivated by the audit.

\paragraph{Does the evidence hypothesis explain the results?}

Figure~\ref{fig:cegarch} shows that the observed improvements
follow the complementary evidence hypothesis. Panel (a)
shows that CEG consistently produces the largest corrections
for detectors that rely primarily on a single evidence
channel, while the protected variant (CEG-P) retains most of
those gains after reducing the strict veto behaviour. Rather
than replacing existing detectors, the guard consistently
improves them by recovering complementary evidence.

Panel (b) explains where these corrections originate.
Among raw MCM errors rescued by CEG, the level channel
provides 81/78/71\% of the vetoes for CLIP/SigLIP2/PE,
whereas spatial sharpness accounts for the remaining
19/22/29\%. This is exactly what Proposition~1 predicts:
because MCM removes absolute level, the dominant
recovered evidence should be level. The remaining
sharpness corrections show that local evidence contributes
independently rather than merely duplicating the global
signal. Together, these results close the loop between the
mechanism and the deployment audit: different detectors
discard different evidence channels, and preserving those
complementary channels improves deployment robustness.

\begin{figure}[t]
\centering
\includegraphics[width=0.97\columnwidth]{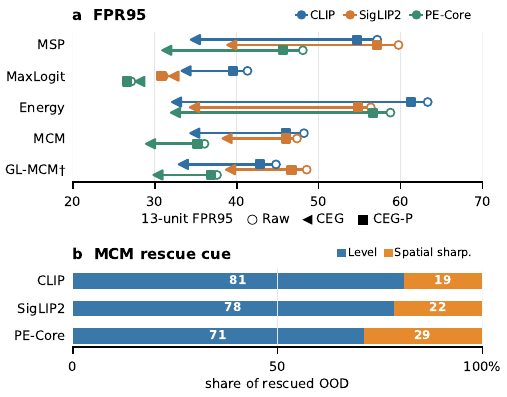}
\caption{\textbf{CEG recovers complementary evidence.}
(a) 13-unit FPR95 across five base detectors and
three VLMs. (b) Most rescued MCM errors originate from
the missing level cue rather than spatial sharpness.}
\label{fig:cegarch}
\end{figure}

% \begin{figure*}[t]
% \centering
% \includegraphics[width=0.99\textwidth]{figures/fig48_ceg_multibase_broad.pdf}
% \caption{\textbf{Compact result and separation view.}
% (a) Raw, \ceg{}, and \cegp{} full-suite FPR95 for all six cached bases and three VLMs.
% (b) Among raw-MCM errors rescued by \ceg{}, level supplies
% $81/78/71\%$ of the vetoes for CLIP/SigLIP2/PE and spatial sharpness the rest;
% the base cannot veto under this conditioning.
% (c) \cegp{} score densities for ID, near OOD, and far OOD, stacked by VLM;
% higher percentiles are more ID-like. Densities average sources within task and
% then 13 display units equally. Panels (b,c) use disclosed seed 0 as an
% interpretation diagnostic; panel (a) averages five ID-only seeds.}
% \label{fig:cegarch}
% \end{figure*}

%%%%% new

\begin{figure}[t]
\centering
\includegraphics[width=\columnwidth]{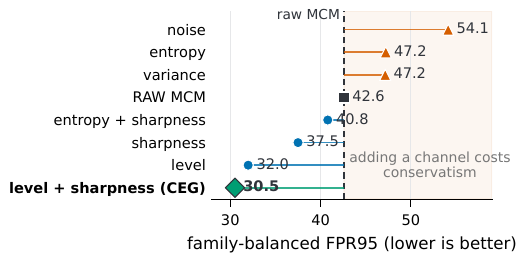}
%\vspace{-2em}
\caption{\textbf{Uninformative channels cost the minimum.}
CLIP-B/16 MCM family-balanced \fpr{} on five strict tasks and fixed ID-only
manifests. Noise is seeded random; entropy and variance independently round to
$47.2$. Shading is worse than raw.}
\label{fig:ladder}
\end{figure}

\paragraph{Channel-substitution controls.}
Adding any channel to a minimum increases conservatism, so we replace level
and sharpness with entropy, logit variance, and seeded noise under identical
calibration and evaluation (Figure~\ref{fig:ladder}). Entropy, variance, and
noise all perform no better than raw MCM ($47.2$, $47.2$, and $54.1$ versus
$42.6$ \fpr{}), whereas sharpness, level, and both reduce \fpr{} to $37.5$,
$32.0$, and $30.5$. The gain therefore follows the information contributed by
the channel, not the conservatism of the minimum. Entropy is largely redundant
with the sharpness evidence already present in MCM, while replacing the
best-anchor level signal with mean affinity improves only to $35.6$, showing
that the useful level evidence is carried by the best-anchor match rather than
average affinity.

\paragraph{Broad deployment audit.}
We next ask whether the same principle holds across deployment
settings. Evaluating \cegp{}, a shrunk variant that retains two
thirds of the base score, on five base detectors and three VLMs,
far OOD improves in all 15 near/far deployment summaries and
both near and far improve simultaneously in 14; every interval
excludes zero. The sole exception is PE MaxLogit on near OOD
($30.77\rightarrow30.92$), which still improves on the
corresponding far benchmark ($24.83\rightarrow23.85$). A
symmetric fold-wise control that scores ID and OOD under
identical protocols preserves all fifteen suite-level gains
within $0.25$ FPR95 points, so the improvements are not an
artefact of the evaluation procedure. Full near/far tables,
deployment-level regressions, and the near/far asymmetry
analysis are in Appendix~\ref{app:broad}.

\section{Conclusion}
\label{sec:conclusion}

We revisited a common assumption in zero-shot OOD detection:
that the detector with the best benchmark performance is also
the best deployment choice. A portability audit across
ID domains and VLMs showed that this assumption
does not hold. Detector rankings reverse because different 
ID domains preserve different evidence, with corpus-free
methods relying primarily on level or sharpness and
WordNet-based methods additionally depending on external
semantic coverage. This diagnosis naturally motivates CEG,
which preserves complementary evidence rather than
committing to a single score, reducing deployment
variability and tail risk without requiring OOD supervision,
auxiliary corpora, or detector-specific adaptation. More
broadly, our results suggest that evaluating zero-shot OOD
detectors should move beyond benchmark averages toward
deployment portability, explicit failure analysis, and
revalidation whenever the ID domain or the underlying VLM
changes.
\FloatBarrier

\bibliography{references,references_v2}

\counterwithin{figure}{section}
\counterwithin{table}{section}
\renewcommand\thefigure{\thesection\arabic{figure}}
\renewcommand\thetable{\thesection\arabic{table}}
\clearpage
\appendix
\makeatletter
\setlength{\@dblfptop}{0pt}
\setlength{\@dblfpsep}{12pt}
\setlength{\@dblfpbot}{0pt plus 1fil}
\makeatother

\section{Appendix Overview and Organization}
\label{app:overview}

This appendix provides the complete numerical evidence and
implementation details behind the main paper. It keeps the
failure audit, the benchmark-composition analysis, the strict
\ceg{} evaluation, and the later protected audit separate so
that their resource assumptions and evidence status remain
explicit. The material is organized as follows:

\begin{itemize}

\item \textbf{Section~\ref{app:protocol}: Detailed reproducibility protocol.}
Frozen components, score and token definitions, datasets,
calibration splits, resource accounting, uncertainty,
aggregation, overlap policy, and fidelity checks. It also
distinguishes stability, prospective, and development
evidence.

\item \textbf{Section~\ref{app:benchmark-choice}: Benchmark composition analysis.}
Additional evidence motivating the use of OpenOOD-v1.5,
including benchmark-overlap examples, nearest-neighbour
analysis, and the comparison with the traditional ImageNet
evaluation protocol.

\item \textbf{Section~\ref{app:audit}: Full failure audit.}
Complete seventeen-domain CLIP results and
Table~\ref{tab:audit}, including the excluded-overlap row
and the mean/worst-domain summaries underlying
Figure~\ref{fig:multibackbone}.

\item \textbf{Section~\ref{app:hard}: Strict and prospective \ceg{} evidence.}
Suite uncertainty, the non-compensatory property of
\ceg{}, channel and combination ablations,
architecture-transfer analysis, the disclosed
nine-checkpoint robustness extension, and the pre-frozen
ImageNet-10/ImageNet-20 confirmation.

\item \textbf{Section~\ref{app:broad}: Protected all-domain audit.}
Near/far accounting, the symmetric cross-fit control,
regression locations, complete Raw/CEG/CEG-P domain
tables, and unaggregated ID$\times$OOD-source results.

\end{itemize}
\section{Detailed Reproducibility Protocol}
\label{app:protocol}

\paragraph{Frozen components and splits.}
All VLMs and score implementations are frozen. The strict \ceg{} study uses
five fixed manifests per task: $1{,}000$ unlabeled train-ID images fit channel
standardisers and empirical CDFs, a disjoint $1{,}000$ train-ID images set the
deployment operating threshold, and official test images remain evaluation
only. The CLIP-to-SigLIP2/PE port freezes tasks, prompts, channel forms,
hyperparameters, and full-image evaluation before reading transfer OOD
results. No experiment fits target-OOD images.

\paragraph{Evidence tiers.}
The channel forms and $k{=}10$, $Q_{0.9}$, $3{\times}3$ kernel, and minimum
fusion were developed on capped versions of the same five strict task domains;
their uncapped rerun is sampling-level stability evidence, not new-task
confirmation. The multi-VLM port and reciprocal ImageNet-10/20 experiment are
prospective. \cegp{} was defined after observing transfer failures and its
18-dataset cross-fit is explicitly post-result development evidence.

\paragraph{Resource and uncertainty accounting.}
Corpus-free global scores use only ID class names; GL-MCM additionally uses
native local tokens, Mahalanobis fits extra unlabeled-ID feature statistics,
and NegLabel/CSP use WordNet. DomainNet counts once in aggregates. Reported
intervals use $5{,}000$ paired image/calibration resamples; broad cellwise
intervals are nominal and unadjusted. The details below give the exact score
equations, datasets, folds, aggregation, overlap policy, and fidelity checks.

\begin{table*}[!t]
\centering
\caption{\textbf{Seventeen-domain CLIP audit.} \fpr{}$\downarrow$
(all-OOD mean, $\%$). Green cells mark the row winner; the header bands expose
resource assumptions. Every detector exceeds $80$ somewhere.
DTD\textsuperscript{\textdagger} is excluded from Mean/Worst because its OOD
suite contains DTD-derived textures.}
\label{tab:audit}
\footnotesize
\renewcommand{\arraystretch}{0.91}
\setlength{\tabcolsep}{3.5pt}
\begin{tabular*}{\textwidth}{@{\extracolsep{\fill}}llrrrrrrr@{}}
\toprule
& & \multicolumn{4}{c}{\cellcolor{clipblue!8}\textbf{Global class-name scores}}
& \cellcolor{black!6}\textbf{Native local}
& \cellcolor{pegreen!10}\textbf{ID fit}
& \cellcolor{sigorange!10}\textbf{WordNet}\\
Family & ID dataset & MSP & MaxLogit & Energy & MCM & GL-MCM & Mahalanobis & NegLabel\\
\midrule
\rowcolor{black!2}ImageNet & ImageNet-1K & 68 & 75 & 97 & 61 & 51 & 96 & \goodcell{47}\\
\rowcolor{black!2}         & ImageNet-200 & 36 & 47 & 95 & 34 & \goodcell{27} & 84 & 32\\
\midrule
Fine-grained & CUB-200 & 77 & \goodcell{2} & 21 & 23 & 47 & \goodcell{2} & \goodcell{2}\\
             & Food-101 & 23 & 6 & 50 & 10 & 6 & \goodcell{3} & \goodcell{3}\\
             & Oxford Pets & 34 & \goodcell{0} & 22 & 8 & 11 & \goodcell{0} & \goodcell{0}\\
             & Caltech-101 & 34 & 34 & 95 & 27 & \goodcell{22} & 27 & 24\\
\midrule
\rowcolor{black!2}Far domains & EuroSAT & 98 & 61 & 22 & 98 & 88 & \goodcell{0} & 38\\
\rowcolor{black!2}            & DTD\raisebox{.45ex}{\textdagger} & 86 & 76 & 83 & 87 & 73 & \goodcell{38} & 51\\
\rowcolor{black!2}            & Pascal VOC & 57 & 58 & 93 & 52 & 38 & 27 & \goodcell{20}\\
\midrule
Renderings & Quickdraw & 97 & 40 & 16 & 97 & 96 & \goodcell{0} & 45\\
           & Clipart & 82 & 97 & 100 & 85 & 82 & \goodcell{37} & \goodcell{37}\\
           & Painting & 89 & 99 & 99 & 90 & 83 & 90 & \goodcell{50}\\
           & Sketch & 88 & 97 & 99 & 88 & 83 & 56 & \goodcell{42}\\
           & Infograph & 90 & 100 & 100 & 87 & 86 & 81 & \goodcell{49}\\
\midrule
\rowcolor{black!2}Small/coarse & CIFAR-10 & 24 & \goodcell{15} & 70 & 18 & 18 & 17 & 36\\
\rowcolor{black!2}             & CIFAR-100 & 87 & 65 & 78 & 83 & 68 & \goodcell{39} & 89\\
\rowcolor{black!2}             & MNIST & 35 & 24 & 33 & 32 & 32 & \goodcell{0} & 70\\
\midrule
\rowcolor{calloutblue!5}\multicolumn{2}{l}{\textbf{Mean (16 rows)}} & 63.7 & 51.3 & 68.1 & 55.8 & 52.4 & \goodcell{34.9} & 36.5\\
\rowcolor{calloutblue!5}\multicolumn{2}{l}{\textbf{Worst (16 rows)}} & 98 & 100 & 100 & 98 & 96 & 96 & \goodcell{89}\\
\bottomrule
\end{tabular*}
\end{table*}

\paragraph{Scores and VLM features.}
ID anchors use the dataset class names with the template
\texttt{a photo of a \{\}.}; NegLabel keeps its original
\texttt{the nice \{\}} template and published NegMining pipeline over roughly
$10$k WordNet negatives. CSP uses the published Conjugated Semantic Pool:
WordNet noun/adjective prompt ensembles are mined against the ID vocabulary,
then grouped to score positive softmax mass. Our deterministic replay sorts
corpus filenames before seeded adjective conjugation. For normalised image, patch, and text embeddings
$g,p_i,t_c$, we write $\ell_c=\langle g,t_c\rangle$ and
$\ell_{i,c}=\langle p_i,t_c\rangle$. The audit implements MSP at the native
logit scale \citep{hendrycks2017baseline}, MaxLogit
\citep{hendrycks2022scaling}, Energy \citep{liu2020energy}, MCM
\citep{ming2022mcm}, GL-MCM \citep{miyai2025glmcm}, and NegLabel
\citep{jiang2024neglabel}. Mahalanobis fits class means and a shared covariance
to zero-shot ID pseudo-labels, never ground-truth labels
\citep{lee2018mahalanobis}. CLIP local scores use projected final-block
MaskCLIP-style value features \citep{zhou2022maskclip,motd2024}. The frozen
transfer port uses SigLIP2's $32{\times}32$ post-normalisation patch sequence
before pooling and PE's projected $14{\times}14$ final sequence without its
class token. Because these native token paths are not representation-identical,
MCM is the architecture-neutral transfer base; GL-MCM is a native-token
diagnostic rather than a representation-equivalent port.

The strict scalar cache predates the multi-VLM audit and lacks the native-logit
scale MSP channel, so MCM is its softmax baseline; MSP is reported throughout
the broad audit where it was frozen. The broad matched set comprises the same
five defined scalar bases for every dataset and VLM at the audit freeze (MSP,
MaxLogit, Energy, MCM, and GL-MCM), not a result-selected subset.
Figure~\ref{fig:multibackbone} keeps the evidence tiers explicit: the
four-score global comparison was frozen before the uncapped rerun, whereas
GL-MCM is shown separately as a post-result native-token extension over the
complete cache.

For Table~\ref{tab:ceg}, Mahalanobis is fitted on $1{,}000$ unlabeled train-ID
images deterministically held out from every guard fit and operating manifest.
This prevents the empirical guard CDF from calibrating on the Gaussian's own
training scores while making the extra ID-fitting resource explicit.

\paragraph{Datasets and OOD pools.}
The CLIP mechanism audit contains ImageNet-1K/200; CUB-200, Food-101, Oxford
Pets, and Caltech-101; EuroSAT, DTD, and Pascal VOC; five non-photographic
DomainNet styles; and CIFAR-10/100 and MNIST
\citep{deng2009imagenet,wah2011cub,bossard2014food,parkhi2012pets,
feifei2004caltech,helber2019eurosat,cimpoi2014dtd,everingham2010voc,
peng2019domainnet,krizhevsky2009cifar,lecun1998mnist}. The uncapped audits also
include DomainNet Real, yielding 18 raw datasets and 13 display units after its
six styles are averaged once. The historical table uses fixed
$2{,}400$-ID/$1{,}500$-OOD draws from the five-source ImageNet OpenOOD-v1.5
pool: SSB-hard, NINCO, iNaturalist, Textures, and OpenImage-O
\citep{zhang2023openood,vaze2022ssb,bitterwolf2023ninco,
vanhorn2018inaturalist,wang2022vim}. Its DTD-as-ID row is excluded from
Mean/Worst because that pool contains DTD-derived OOD images. The uncapped
audits omit Textures from all shared ImageNet reference probes, removing this
self-overlap; CIFAR and MNIST retain their native near/far suites. Shared probes
hold OOD inputs fixed while ID changes. Beyond removing exact DTD/Textures
self-overlap, we do not filter semantic class overlap because the shared probe
sources lack a complete label mapping to each ID vocabulary. These rows permit
category overlap by design and are reference-shift stress tests, not strict
semantic-novelty benchmarks. Rendering rows likewise test category plus
rendition shift.

\paragraph{Primary guard and prospective transfer.}
The strict suite is ImageNet-1K, ImageNet-200, CIFAR-10, CIFAR-100, and MNIST.
The level/sharpness forms, $k{=}10$, $Q_{0.9}$, $3{\times}3$ kernel, and minimum
fusion were developed on capped versions of these same five task domains.
Freezing them before the uncapped rerun prevents image-level leakage but does
not make the strict headline task-level confirmatory; only the reciprocal
ImageNet-10/20 check and prospective VLM port are out-of-development tests.
For each task, five fixed manifests use $1{,}000$ unlabeled train-ID images to
fit channel standardisers/CDFs and a disjoint $1{,}000$ train-ID images for the
deployment operating threshold; all official test images are reserved for
evaluation. The reported benchmark \fpr{} convention instead sets the
$95\%$-recall threshold on evaluation ID. The primary aggregate first averages
within ImageNet and CIFAR and then weights ImageNet, CIFAR, and MNIST equally.
Before reading SigLIP2/PE OOD results, we froze the tasks, manifests, prompt,
guard hyperparameters, and full-image policy; no architecture-specific layer,
token, task, projection, or weight sweep was performed. Paired replicates share
physical-image and calibration-manifest draws across VLMs.

\paragraph{Post-result protected audit.}
The broader \cegp{} study is deliberately labelled development evidence. For
datasets without a designated calibration split, and uniformly throughout
this audit, we use five-seed, five-fold ID-only cross-fitting. In each fold a
guard is fitted on a deterministic sample of up to $1{,}000$ images from the
other four folds; if fewer are available, it uses every non-held-out ID image.
Thus every ID image receives a score from a guard that did not fit that image.
For example, VOC-13 has 960 images, so each guard fits approximately 768 and
scores the remaining approximately 192. Each OOD image receives the mean score
over the five guards, and no OOD image enters guard fitting. ID labels only
stratify fold assignment.
Thus this cross-fit can use more unique ID images than the primary calibration
and applies different ensemble constructions to ID and OOD. The same paired
physical-image and cross-fit draws are shared across scores, tasks, variants,
and VLMs; external-corpus scores are excluded from this matched comparison.
The symmetric fold-wise control in Appendix~\ref{app:broad} removes this
ID/OOD scoring asymmetry. Mahalanobis is also excluded: unlike the common
cached score channels, it
requires a backbone-specific ID-feature fit, and equivalent global-feature
caches were not retained for all three VLMs. CSP is excluded because its
external mined prompt pool is not resource-matched across the three VLMs.

\paragraph{Sanity checks.}
On the classic iNaturalist/SUN/Places/Textures ImageNet protocol
\citep{xiao2010sun,zhou2017places}, our MCM and NegLabel implementations obtain
$42.7$ and $26.0$ \fpr{}, respectively, versus published $43.9$ and $25.4$.
Zero-shot top-1 accuracy on the strict suite is $67.09$, $91.77$, $89.73$,
$66.70$, and $38.43\%$ in task order. The low MNIST accuracy is a caveat:
better rejection alone does not make the classifier deployment-ready.

\section{Why We Use OpenOOD v1.5}
\label{app:benchmark-choice}

\paragraph{The benchmark label does not guarantee sample-level separation.}
The traditional zero-shot ImageNet protocol established by MOS
\citep{huang2021mos} uses iNaturalist, SUN, Places, and Textures as four
far-OOD sources. Figure~\ref{fig:classic-overlap}a shows
selected high-similarity cross-dataset pairs. Three illustrate shared visual
domains rather than proven leakage. The Places case is qualitatively different:
\texttt{i\_igloo\_00004659.jpg} and ImageNet validation image
\texttt{ILSVRC2012\_val\_00042483.JPEG} contain the same aligned
$600{\times}480$ photograph. Their pixels differ only at encoding level
(PSNR $51.5$ dB; $72.9\%$ of RGB channel values are exactly equal), making
this a verified cross-split duplicate rather than merely a semantic neighbour.

\paragraph{A matched nearest-neighbour diagnostic.}
To test whether visual proximity extends beyond hand-picked examples, we
stratify ImageNet validation within each of its $1{,}000$ classes (seed 0): 25 images
per class form a 25k gallery and the remaining 25 form a held-out ID query
set. For each held-out ID image and 5k images from each traditional OOD source,
we compute its maximum cosine similarity to the same gallery using frozen CLIP
ViT-B/16 image embeddings. At the threshold retaining $95\%$ of held-out ID,
$93.7$--$99.3\%$ of these nominally OOD images also pass
(Figure~\ref{fig:classic-overlap}b). This one-nearest-neighbour probe is a
visual-domain diagnostic, not an estimate of duplicate prevalence and not a
claim that the datasets are impossible to distinguish. It shows that visual
proximity alone does not cleanly realize the intended far-OOD boundary.

\begin{figure*}[h]
\centering
\includegraphics[width=0.99\textwidth]{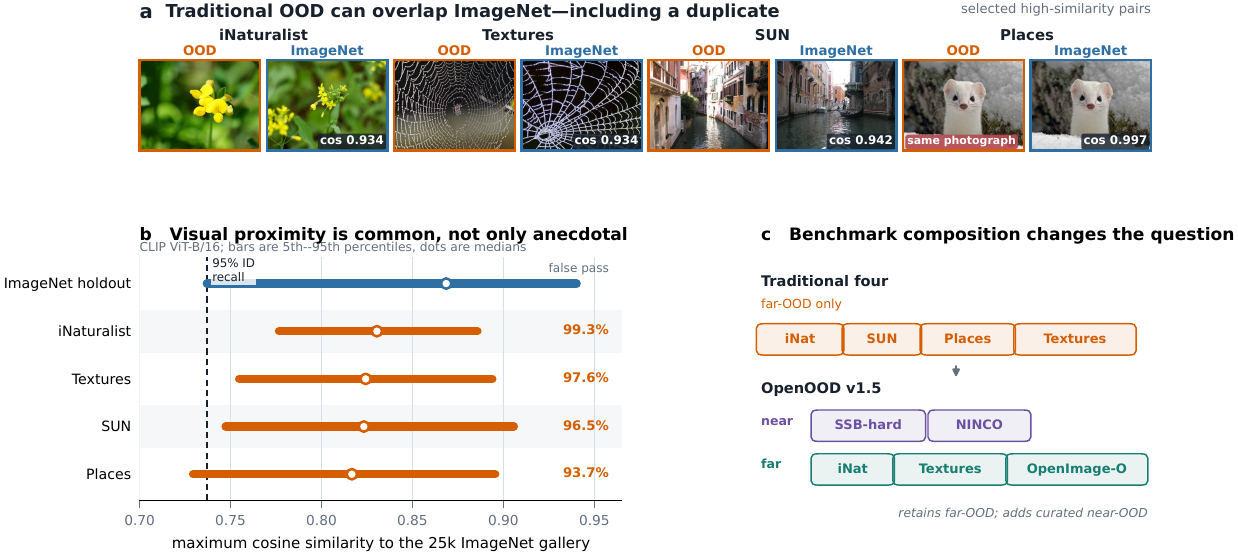}
\caption{\textbf{One label, a wide range of proximity.} (a) Pairs
from the traditional four-source ImageNet suite, ordered by
similarity: all are labelled far-OOD, yet they range from shared
visual domain to a verified re-encoded duplicate of an ImageNet
validation image. (b) Maximum cosine similarity to a matched 25k
ImageNet gallery. Bars span the 5th--95th percentiles and circles mark medians;
the dashed line is the threshold retaining $95\%$ of the held-out ImageNet
queries. Orange percentages are the OOD false-pass rates at that threshold.
(c) The traditional protocol exposes only four far-OOD sources. OpenOOD-v1.5
retains iNaturalist and Textures, adds OpenImage-O, and separately exposes the
curated near-OOD SSB-hard and NINCO sources.}
\label{fig:classic-overlap}
\end{figure*}

\subsection{Audit Protocol}
\label{sec:setup}

\paragraph{Models and detectors.}
The audit separates mechanism from portability. We therefore
study the evidence mechanism using frozen CLIP ViT-B/16
\citep{radford2021clip}, then replicate the conclusions on
SigLIP2-B/16 \citep{tschannen2025siglip2} and PE-Core-B/16
\citep{bolya2025pe} with their native preprocessing. The
comparison includes MSP, MaxLogit, Energy, Mahalanobis,
MCM, GL-MCM and NegLabel, before
later wrapping ten structurally different detectors with the
same guard. ID class names provide the text anchors, while
NegLabel retains its published WordNet pipeline.
Mahalanobis requires an additional ID-feature fit and is
therefore reported as a separate resource tier. Exact score
definitions are given in Appendix~\ref{app:protocol}.

\paragraph{Deployment settings.}
The audit combines a 17-dataset CLIP mechanism study with an
uncapped 18-dataset, three-VLM replication spanning five
deployment families. Uncapped experiments retain every
evaluation image and each dataset's native vocabulary;
CIFAR-10/100 and MNIST use their native OpenOOD suites,
while the shared ImageNet reference pool is fixed before
scoring. This design isolates detector portability across
both deployment domains and model families.

\paragraph{Evaluation criteria.}
Detector portability is measured using
\fpr{}$\downarrow$ at $95\%$ ID recall, averaged across each
suite's OOD splits. DomainNet contributes once to
cross-domain aggregates. Unless stated otherwise, all
experiments are fully zero-shot and use frozen models,
without fitting to OOD images or ground-truth OOD labels.
Complete implementation details are given in
Appendix~\ref{app:protocol}.

% \paragraph{Guard protocol.}
% The guard is evaluated under a strict separation between
% development and validation. The primary \ceg{} experiment
% uses five OpenOOD tasks with fixed unlabeled train-ID
% calibration manifests and held-out test sets. The channel
% forms, $k{=}10$, $Q_{0.9}$, the $3{\times}3$ kernel, and the
% minimum fusion were developed only on capped versions of
% these same tasks, so the uncapped evaluation is not a
% task-level confirmation. The CLIP-to-SigLIP2/PE transfer
% freezes both tasks and hyperparameters. \cegp{} serves as
% post-result validation through five-fold cross-fitting over
% 18 raw ID datasets (13 deployment units), using labels only
% for fold stratification. Confidence intervals use $5{,}000$
% paired resamples. Neither guard is fit on OOD images or
% ground-truth labels. Complete protocols appear in
% Appendix~\ref{app:protocol}.

\section{Full Seventeen-Domain Audit}
\label{app:audit}

\paragraph{How to read the audit.}
The mean and worst rows answer different deployment questions. Mahalanobis has
the lowest cross-domain mean, but its ImageNet failures leave a $96$ worst
case; NegLabel has the best worst case, yet fails badly on CIFAR-100 and MNIST.
The resource bands make this trade-off explicit: an ID-feature fit or external
corpus can move the failure rather than remove it. The within-family reversals
also rule out a simple ``photographic versus non-photographic'' explanation;
even related fine-grained datasets reward different scores.

\paragraph{Evidence boundary.}
Table~\ref{tab:audit} contains raw detector scores only; no row uses \ceg{} or
\cegp{}. DTD is shown for transparency but excluded from Mean/Worst because
its declared OOD pool contains DTD-derived textures. Section~\ref{app:hard}
next evaluates the fixed minimum guard, while Section~\ref{app:broad} keeps the
later protected development study separate.

\FloatBarrier

\section{CEG Evidence and Uncertainty}
\label{app:hard}

\subsection{Non-compensatory property of CEG}
\label{app:ceg-proof}

Equation~(\ref{eq:ceg}) defines the guard as

\[
G_B(x)
=
\min_{R\in\{B,L,S\}}
U_R(x).
\]

The minimum has the following elementary property.

\begin{proposition}
For every $\tau\in[0,1]$,

\[
\{
G_B\ge\tau
\}
=
\bigcap_{R\in\{B,L,S\}}
\{
U_R\ge\tau
\},
\]

and therefore, for any test distribution $Q$,

\[
Q(G_B\ge\tau)
\le
\min_R
Q(U_R\ge\tau).
\]

\end{proposition}

\begin{proof}

A minimum is at least $\tau$ if and only if every argument
is at least $\tau$. The probability inequality follows
immediately because the intersection is contained in each
individual event.

\end{proof}

This proposition formalises the non-compensatory behaviour
of CEG. It explains why a single atypical evidence channel
can veto acceptance, but it does not imply improved OOD
performance after calibration. The empirical benefits of
this design are evaluated in Section~\ref{sec:guard}.

\paragraph{Image-level intuition.}
Figure~\ref{fig:imageintuition} illustrates the
non-compensatory behaviour formalised by
Proposition~2. The texture matches the ID classes strongly
in absolute level but lacks a coherent class peak, whereas
the street scene has a sharp ``truck'' winner despite an
atypically weak absolute match. Because the missing cue
changes across samples, preserving both channels allows
either one to veto acceptance when it becomes atypical.

\begin{figure}[t]
\centering
\includegraphics[width=\columnwidth]{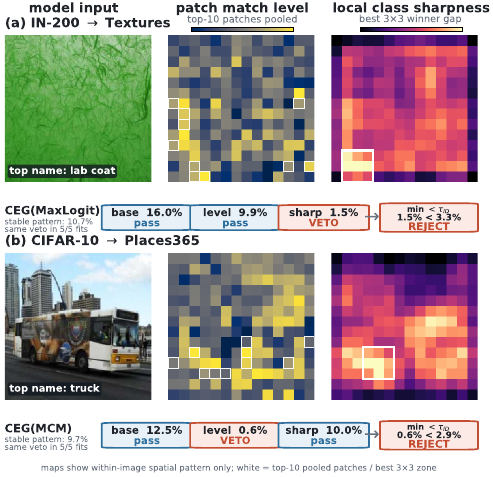}
\caption{\textbf{Image-level non-compensation in \ceg{}.}
Representative conditional-median examples show opposite
vetoes. A texture passes the base detector and level but
fails sharpness, whereas a bus passes the base detector and
sharpness but fails level. These patterns are stable across
all five frozen ID-only fits.}
\label{fig:imageintuition}
\end{figure}

% AUTO-GENERATED by scripts/104_make_ceg_paper_tables.py
\begin{center}
\begin{minipage}{\columnwidth}
\centering\scriptsize
\setlength{\tabcolsep}{4.0pt}
\captionof{table}{\textbf{Suite-level uncertainty for CEG.} $\Delta$ is CEG minus Raw family-balanced FPR95; every interval is below zero and is marked green. The intervals use 5,000 paired image/calibration resamples and summarize the whole five-task suite, not an individual domain.}
\label{tab:ceg-ci}
\begin{tabular}{lrr}
\toprule
Base & $\Delta$FPR95 & Paired 95\% interval\\
\midrule
MaxLogit & \goodcell{-10.16} & \goodcell{[-10.95, -9.56]}\\
Energy ($\tau{=}1$) & \goodcell{-35.93} & \goodcell{[-36.86, -35.08]}\\
MCM & \goodcell{-12.07} & \goodcell{[-12.99, -11.49]}\\
GL-MCM$^\dagger$ & \goodcell{-9.28} & \goodcell{[-9.92, -8.73]}\\
20\%-nonmax gap (abl.) & \goodcell{-28.93} & \goodcell{[-30.09, -27.99]}\\
LogitGap & \goodcell{-19.39} & \goodcell{[-20.28, -18.44]}\\
Mahalanobis & \goodcell{-18.22} & \goodcell{[-18.83, -17.80]}\\
NegLabel & \goodcell{-27.20} & \goodcell{[-28.23, -25.90]}\\
CSP & \goodcell{-21.17} & \goodcell{[-22.08, -20.10]}\\
\bottomrule
\end{tabular}
\end{minipage}
\end{center}

\begin{table*}[t]
\centering\small
\setlength{\tabcolsep}{5.5pt}
\renewcommand{\arraystretch}{1.02}
\begin{tabular*}{\textwidth}{@{\extracolsep{\fill}}lrrrrr}
\toprule
Base detector & Raw & \cegp{} ($\lambda{=}1/3$) & $\lambda{=}0.50$ & $\lambda{=}0.75$ & Hard \ceg{} \\
\midrule
\multicolumn{6}{l}{\textbf{Class-name scores; no external corpus}}\\
MaxLogit & 48.58 & \goodcell{45.47} & \goodcell{43.53} & \goodcell{39.76} & \goodcell{36.92} \\
Energy ($\tau{=}1$) & 84.79 & \goodcell{83.20} & \goodcell{82.03} & \goodcell{78.04} & \goodcell{39.95} \\
MCM & 48.49 & \goodcell{45.67} & \goodcell{44.25} & \goodcell{41.37} & \goodcell{37.50} \\
GL-MCM$^\dagger$ & 41.50 & \goodcell{40.09} & \goodcell{39.22} & \goodcell{37.59} & \goodcell{35.99} \\
20\%-nonmax gap$^\ddagger$ & 52.73 & \goodcell{48.75} & \goodcell{46.65} & \goodcell{42.37} & \goodcell{37.59} \\
\midrule
\multicolumn{6}{l}{\textbf{External WordNet corpus/mining}}\\
NegLabel & 50.08 & \goodcell{47.61} & \goodcell{45.95} & \goodcell{42.88} & \goodcell{35.62} \\
CSP-deterministic & 44.49 & \goodcell{41.66} & \goodcell{40.34} & \goodcell{38.15} & \goodcell{33.16} \\
\bottomrule
\end{tabular*}

\caption{\textbf{Protected CEG deployment frontier on native OpenOOD, excluding MNIST.}
All entries are mean FPR95$\downarrow$ with CLIP-B/16. OOD sources are first
averaged within each task, after which ImageNet-1K, ImageNet-200, CIFAR-10, and
CIFAR-100 receive equal weight. This is a post-result sensitivity view:
MNIST remains in the complete five-task analyses but is excluded here so that
its qualitatively distinct coarse digit regime does not dominate this mean.
The declared \cegp{} point uses $\lambda=1/3$; $\lambda=0.50$ and $0.75$ are
additional fixed post-result deployment points, and $\lambda=1$ is Hard
\ceg{}. No point is selected from these OOD results. Every displayed nonraw
point improves the four-task mean for all seven bases.}
\label{tab:ceg-openood-frontier}
\vspace{1pt}\par\tiny $^\dagger$Uses native local features. $^\ddagger$Represents the numerically equivalent fixed-$N$ LogitGap family in this protection sweep.
\end{table*}

\paragraph{Protected frontier across native OpenOOD.}
Table~\ref{tab:ceg-openood-frontier} reports the deployment choice on the four
native non-digit tasks while retaining MNIST in the complete five-task
analyses. Every displayed nonraw point improves the mean for all seven bases.
This includes the external-corpus detectors: NegLabel moves from $50.08$ to
$47.61$, $45.95$, $42.88$, and $35.62$ as $\lambda$ increases, while CSP moves
from $44.49$ to $41.66$, $40.34$, $38.15$, and $33.16$. Thus the aggregate
conclusion does not rely on omitting strong external-semantic baselines. The
frontier is descriptive post-result evidence, not an OOD-selected operating
point.

% AUTO-GENERATED by scripts/107_ceg_channel_combination_ablation.py
\begin{center}
\begin{minipage}{\columnwidth}
\centering\scriptsize
\setlength{\tabcolsep}{2.1pt}
\captionof{table}{\textbf{Channel and combination ablation.} Family-balanced FPR95$\downarrow$ on the five strict CLIP tasks. Green/red denotes lower/higher FPR95 than the gray Raw reference by more than $0.10$ points; gray denotes a practical tie. ``Global'' removes local-token cues; the last column averages all five corpus-free bases. This is same-task development evidence.}
\label{tab:ceg-ablation}
\begin{tabular*}{\columnwidth}{@{\extracolsep{\fill}}lrrrrr}
\toprule
Rule & MCM & MaxL. & Energy & GL & All\\
\midrule
Raw base & \neutralcell{42.6} & \neutralcell{40.2} & \neutralcell{68.0} & \neutralcell{38.1} & \neutralcell{49.9}\\
\addlinespace[1pt]
Base $\wedge$ MaxLogit & \goodcell{35.8} & \badcell{40.7} & \goodcell{47.8} & \goodcell{31.6} & \goodcell{39.4}\\
Base $\wedge$ level only & \goodcell{32.0} & \goodcell{33.7} & \goodcell{37.1} & \goodcell{29.1} & \goodcell{33.3}\\
Base $\wedge$ sharpness only & \goodcell{37.5} & \goodcell{31.2} & \goodcell{35.6} & \goodcell{35.7} & \goodcell{36.4}\\
Global-only guard & \goodcell{35.8} & \goodcell{35.8} & \goodcell{40.1} & \goodcell{32.2} & \goodcell{36.4}\\
\addlinespace[1pt]
\textbf{Full CEG (minimum)} & \goodcell{30.5} & \goodcell{30.1} & \goodcell{32.1} & \goodcell{28.8} & \goodcell{30.6}\\
\addlinespace[1pt]
Mean of three ranks & \goodcell{32.6} & \goodcell{31.9} & \goodcell{47.7} & \goodcell{31.6} & \goodcell{35.3}\\
Simes & \goodcell{29.9} & \goodcell{29.1} & \goodcell{31.2} & \goodcell{28.2} & \goodcell{29.9}\\
Fisher & \goodcell{30.0} & \goodcell{29.1} & \goodcell{29.9} & \goodcell{28.6} & \goodcell{29.6}\\
\bottomrule
\end{tabular*}
\end{minipage}
\end{center}

\paragraph{Mean gain is bought with tail risk.}
Table~\ref{tab:ceg-lambda-frontier} sweeps the shrinkage weight
$\lambda$ from the raw base to the hard minimum and beyond. Within
the convex regime the two columns move together: larger mean
corrections come with larger worst-case regressions, from
$-0.41/{+}0.43$ at $\lambda=0.1$ to $-5.71/{+}16.04$ at the hard
guard, which also regresses on $55$ of $135$ task cells.
\cegp{} ($\lambda=1/3$) sits where the trade is most favourable:
it improves all $27$ checkpoint aggregates while bounding the worst
task regression at $+1.01$. Past $\lambda=1$ the score is no longer
a convex combination, and vetoing more aggressively fails outright
($+52.16$ mean, $+96.86$ worst at $\lambda=2$). Reliability
therefore does not increase monotonically with conservatism; as in
the channel substitutions of Figure~\ref{fig:ladder}, it depends on
how much evidence the added veto actually carries.

% AUTO-GENERATED by scripts/110_ceg_lambda_frontier.py
\begin{table}[t]
\centering\small
\setlength{\tabcolsep}{3.5pt}
\resizebox{\columnwidth}{!}{%
\begin{tabular}{llrrrc}
\toprule
$\lambda$ & Regime & Mean $\Delta$ & Worst $\Delta$ & Task I/T/W & Agg. I/T/W \\
\midrule
$0$ & Raw base & \neutralcell{$+0.00$} & \neutralcell{$+0.00$} & 0/135/0 & 0/27/0 \\
$0.1$ & Interpolation & \goodcell{$-0.41$} & \badcell{$+0.43$} & 95/29/11 & 26/1/0 \\
$0.2$ & Interpolation & \goodcell{$-0.82$} & \badcell{$+0.56$} & 102/19/14 & 27/0/0 \\
$1/3$ & CEG-P & \goodcell{$-1.47$} & \badcell{$+1.01$} & 110/11/14 & 27/0/0 \\
$0.5$ & Interpolation & \goodcell{$-2.34$} & \badcell{$+2.16$} & 112/6/17 & 27/0/0 \\
$0.75$ & Interpolation & \goodcell{$-4.16$} & \badcell{$+5.91$} & 103/8/24 & 27/0/0 \\
$1$ & Hard CEG & \goodcell{$-5.71$} & \badcell{$+16.04$} & 76/4/55 & 22/0/5 \\
\midrule
$1.25$ & Amplification & \badcell{$+23.89$} & \badcell{$+96.65$} & 26/0/109 & 3/0/24 \\
$1.5$ & Amplification & \badcell{$+48.56$} & \badcell{$+96.81$} & 9/0/126 & 0/0/27 \\
$2$ & Amplification & \badcell{$+52.16$} & \badcell{$+96.86$} & 3/0/132 & 0/0/27 \\
\bottomrule
\end{tabular}
}
\caption{\textbf{CEG deployment risk--reward frontier.} $S_\lambda=U_B-\lambda(U_B-\min\{U_B,U_L,U_S\})$ over 3 bases, 9 checkpoints, and 5 strict tasks. $\Delta$ is wrapped minus raw FPR95, so negative is better. I/T/W counts improvements, practical ties ($|\Delta|\le0.10$), and regressions. ``Worst'' is the largest task-level regression. Green/red/gray denotes improvement/regression/practical tie to Raw in the two $\Delta$ columns. The complete fixed grid is post-result descriptive evidence; no $\lambda$ is selected from these OOD results. $\lambda=1/3$ and $1$ are the declared CEG-P and hard-CEG references. Amplification ($\lambda>1$) is outside the protected convex regime and fails sharply.}
\label{tab:ceg-lambda-frontier}
\end{table}

\paragraph{Capacity-matched combination control.}
Table~\ref{tab:ceg-ablation} gives every combination rule exactly the same
three ranks $U_B,U_L,U_S$. Their compensatory mean reaches $35.3$ FPR95,
versus $30.6$ for the minimum, so the gain is not explained merely by exposing
the detector to three channels. Simes and Fisher reach $29.9$ and $29.6$ on
these same development tasks, so their numerical advantage is real. It answers
a different design objective: Fisher can offset one atypical rank with typical
ranks, while Simes discounts an isolated atypical rank unless multiple
channels agree. Both relax the single-channel veto and can therefore suppress
the missing-cue alarm identified by the audit. We retain the minimum as the
declared non-compensatory rule, not as the post-hoc numerical optimum.

\paragraph{Tie-margin sensitivity.}
In the nine-checkpoint deployment frontier, at $\lambda=1/3$, task I/T/W
counts are $114/4/17$, $110/11/14$, and $105/22/8$ for margins $0.05$, $0.10$,
and $0.20$, respectively. All $27/27$ checkpoint/method aggregates improve
at every margin; the same $27/27$ aggregate conclusion holds for
$\lambda=0.50$ and $0.75$.

\paragraph{Protected-weight provenance.}
The $2{:}1$ ratio was chosen after inspecting the transfer failures it was
designed to mitigate; it is not a held-out hyperparameter choice, and no
\cegp{} comparison is presented as confirmation. Every \cegp{} table and the
fixed $\lambda$ frontier are explicitly post-result development evidence. The
score uses no OOD sample during fitting or inference, but that resource
property does not make the weight selection prospective.

% \input{tables/tab_ceg_lambda_frontier}
% AUTO-GENERATED by scripts/109_make_ceg_checkpoint_table.py
\begin{table*}[t]
\centering\scriptsize
\setlength{\tabcolsep}{3.0pt}
\renewcommand{\arraystretch}{0.97}
\caption{\textbf{CEG-P and Hard CEG across three detector bases and nine checkpoints.} Each family row equally averages its three checkpoints; every checkpoint value equally averages the five strict tasks after within-task OOD-source aggregation. All numeric scores are FPR95$\downarrow$, and $\Delta$ is wrapped minus Raw, so negative is better. I/T/W gives improvements, practical ties, and regressions; counts before and after the semicolon refer to checkpoint aggregates and task cells. CEG-P improves 27/27 detector/checkpoint aggregates, whereas Hard CEG improves 22 and worsens 5. The evaluated bases are MCM, MaxLogit, GL-MCM$^\dagger$. Energy is absent because its native-scale score was not retained for all six added checkpoints, although it is available for the original CLIP-B/16 strict evaluation in Table~\ref{tab:ceg}; NegLabel and CSP are absent because their external-corpus scores were not recomputed for those checkpoints. Methods are unavailable rather than result-filtered. The extension is post-result robustness evidence.}
\label{tab:ceg-checkpoint-extension}
\begin{tabular*}{\textwidth}{@{\extracolsep{\fill}}llrrrlrrl@{}}
\toprule
VLM family & Base detector $B$ & Raw & \cegp{} & $\Delta_{\rm P}$ & P I/T/W (agg.; task) & Hard \ceg{} & $\Delta_{\rm H}$ & H I/T/W (agg.; task)\\
\midrule
\rowcolor{clipblue!7}\multicolumn{9}{l}{\textbf{CLIP}}\\
CLIP & MCM & 45.83 & 43.27 & \goodcell{-2.55} & 3/0/0; 14/1/0 & 32.91 & \goodcell{-12.92} & 3/0/0; 11/0/4\\
CLIP & MaxLogit & 38.78 & 36.89 & \goodcell{-1.89} & 3/0/0; 15/0/0 & 31.95 & \goodcell{-6.83} & 3/0/0; 12/0/3\\
CLIP & GL-MCM$^\dagger$ & 45.76 & 43.48 & \goodcell{-2.28} & 3/0/0; 14/0/1 & 32.52 & \goodcell{-13.24} & 3/0/0; 11/0/4\\
\addlinespace[1pt]
\rowcolor{sigorange!8}\multicolumn{9}{l}{\textbf{SigLIP2}}\\
SigLIP2 & MCM & 59.61 & 58.09 & \goodcell{-1.52} & 3/0/0; 10/4/1 & 47.37 & \goodcell{-12.24} & 3/0/0; 9/0/6\\
SigLIP2 & MaxLogit & 32.31 & 31.52 & \goodcell{-0.79} & 3/0/0; 12/1/2 & 33.85 & \badcell{+1.54} & 0/0/3; 5/1/9\\
SigLIP2 & GL-MCM$^\dagger$ & 61.72 & 59.71 & \goodcell{-2.01} & 3/0/0; 12/2/1 & 49.20 & \goodcell{-12.53} & 3/0/0; 9/0/6\\
\addlinespace[1pt]
\rowcolor{pegreen!7}\multicolumn{9}{l}{\textbf{PE-Core}}\\
PE-Core & MCM & 38.82 & 37.38 & \goodcell{-1.44} & 3/0/0; 12/1/2 & 33.70 & \goodcell{-5.12} & 2/0/1; 6/0/9\\
PE-Core & MaxLogit & 29.93 & 28.94 & \goodcell{-0.99} & 3/0/0; 11/1/3 & 28.75 & \goodcell{-1.18} & 2/0/1; 8/1/6\\
PE-Core & GL-MCM$^\dagger$ & 41.41 & 40.44 & \goodcell{-0.97} & 3/0/0; 10/1/4 & 36.14 & \goodcell{-5.27} & 3/0/0; 5/2/8\\
\addlinespace[1pt]
\midrule
\multicolumn{9}{l}{\textbf{All nine checkpoints}}\\
All & MCM & 48.09 & 46.25 & \goodcell{-1.84} & 9/0/0; 36/6/3 & 37.99 & \goodcell{-10.09} & 8/0/1; 26/0/19\\
All & MaxLogit & 33.67 & 32.45 & \goodcell{-1.22} & 9/0/0; 38/2/5 & 31.52 & \goodcell{-2.16} & 5/0/4; 25/2/18\\
All & GL-MCM$^\dagger$ & 49.63 & 47.88 & \goodcell{-1.75} & 9/0/0; 36/3/6 & 39.29 & \goodcell{-10.34} & 9/0/0; 25/2/18\\
\bottomrule
\end{tabular*}
\vspace{1pt}\par\tiny $^\dagger$GL-MCM uses native local tokens. Practical ties satisfy $|\Delta|\le0.10$ and are descriptive rather than statistical equivalence.
\end{table*}

% AUTO-GENERATED by scripts/104_make_ceg_paper_tables.py
\begin{table*}[t]
\centering\scriptsize
\setlength{\tabcolsep}{4.0pt}
\renewcommand{\arraystretch}{0.90}
\caption{\textbf{Aggregation robustness across evaluation weightings.} Each cell reports Raw$\to$\ceg{} FPR95$\downarrow$/AUROC$\uparrow$ for five base detectors. For MaxLogit, Energy, MCM, and GL-MCM, Fam. and no-MN use the final strict protocol: Fam. balances ImageNet, CIFAR, and MNIST, while no-MN excludes MNIST and rebalances ImageNet and CIFAR. The 13-unit column instead uses the independently frozen all-image broad cache of Tables~\ref{tab:ceg-all-0}--\ref{tab:ceg-all-2}, with every unit weighted equally and DomainNet counted once. MSP is absent from the strict cache, so all three MSP cells use the broad all-image protocol. Comparisons should be made within a column.}
\label{tab:ceg-weighting-sensitivity}
\begin{tabular*}{\textwidth}{@{\extracolsep{\fill}}lrrr}
\toprule
Base & Fam. & 13-unit & no-MN\\
\midrule
MSP & \goodnum{48.4$\!\to\!$30.7\,/\,86.3$\!\to\!$91.8} & \goodnum{57.2$\!\to\!$35.0\,/\,78.5$\!\to\!$88.8} & \goodnum{53.9$\!\to\!$38.0\,/\,84.3$\!\to\!$89.3}\\
\addlinespace[0.35pt]
MaxLogit & \goodnum{40.2$\!\to\!$30.1\,/\,89.8$\!\to\!$92.2} & \goodnum{41.3$\!\to\!$33.8\,/\,86.5$\!\to\!$89.2} & \goodnum{48.6$\!\to\!$37.2\,/\,87.1$\!\to\!$89.8}\\
\addlinespace[0.35pt]
Energy & \goodnum{68.0$\!\to\!$32.1\,/\,67.5$\!\to\!$91.3} & \goodnum{63.3$\!\to\!$32.7\,/\,69.8$\!\to\!$90.0} & \goodnum{84.8$\!\to\!$40.1\,/\,56.0$\!\to\!$88.4}\\
\addlinespace[0.35pt]
MCM & \goodnum{42.6$\!\to\!$30.5\,/\,88.7$\!\to\!$92.1} & \goodnum{48.2$\!\to\!$34.9\,/\,81.7$\!\to\!$89.0} & \goodnum{48.5$\!\to\!$37.8\,/\,86.3$\!\to\!$89.6}\\
\addlinespace[0.35pt]
GL-MCM$^\dagger$ & \goodnum{38.1$\!\to\!$28.8\,/\,89.4$\!\to\!$92.3} & \goodnum{44.8$\!\to\!$33.5\,/\,85.0$\!\to\!$89.3} & \goodnum{41.5$\!\to\!$36.2\,/\,87.6$\!\to\!$89.7}\\
\addlinespace[0.35pt]
\bottomrule
\end{tabular*}
\vspace{1pt}\par\tiny $^\dagger$Native local features; the comparison matches the five-base main-paper scope.
\end{table*}

\paragraph{Prospective transfer details.}
For frozen MCM+\ceg{}, the paired family-macro changes are $-12.07$
[$-12.99,-11.48$] for CLIP, $-14.27$ [$-15.49,-13.38$] for SigLIP2, and
$-8.60$ [$-9.11,-7.85$] for PE. Both SigLIP2 ImageNet tasks worsen, as do PE
ImageNet-200, CIFAR-10, and CIFAR-100: five of 15 point estimates. Four paired
intervals lie above zero and PE CIFAR-100 remains unresolved. \ceg{}(MaxLogit)
also worsens the SigLIP2 aggregate from $28.77$ to $30.25$.

\paragraph{Post-result checkpoint extension.}
Table~\ref{tab:ceg-checkpoint-extension} crosses three CLIP, three SigLIP2, and
three PE-Core checkpoints with every detector base available throughout the
strict cache: MCM, MaxLogit and GL-MCM. Tasks, images, calibration
manifests, prompts, and guard hyperparameters remain fixed. \cegp{} improves all
27 detector/checkpoint aggregates and 110/135 task cells; its nine
detector/family mean changes range from $-0.79$ to $-2.55$. Hard \ceg{} produces
larger mean gains but improves only 22/27 aggregates and 76/135 task cells.
Its clearest failure is MaxLogit on SigLIP2, whose family mean worsens by
$1.54$. Thus protection generalizes beyond MCM and better tolerates
architecture-dependent channel failures. Because the six additional
checkpoints were chosen after the original transfer was read, the extension is
robustness evidence rather than prospective confirmation.

\paragraph{Pre-frozen near-OOD confirmation.}
Under the pre-frozen reciprocal ImageNet-10/20 protocol, seven of eight
intervals exclude zero: MCM, GL-MCM, MaxLogit, Energy, fixed-gap, LogitGap,
and NegLabel. CSP ($[-1.75,1.25]$) is unresolved.
Because these are correlated ImageNet subsets, the multi-family suite remains
primary.

\FloatBarrier

\section{Protected All-Domain Audit}
\label{app:broad}

% AUTO-GENERATED by scripts/103_plot_broad_ceg_methods.py
\begin{table}[t]
\centering\scriptsize
\setlength{\tabcolsep}{1.5pt}
\renewcommand{\arraystretch}{0.94}
\caption{\textbf{Near- and far-OOD point macros for the broad protected audit.} Entries are Raw$\to{}$\cegp{} FPR95$\downarrow$; $\Delta$ is wrapped minus Raw. Gray marks a practical tie, $|\Delta|\leq0.10$. Each value equally averages the 13 display units (DomainNet counts once). Full-suite paired intervals are in Table~\ref{tab:ceg-ci}.}
\label{tab:cegp-nearfar-full}
\begin{tabular*}{\columnwidth}{@{\extracolsep{\fill}}llrrrr@{}}
\toprule
& & \multicolumn{2}{c}{Near OOD} & \multicolumn{2}{c}{Far OOD}\\
VLM & Base & Raw$\to{}$\cegp{} & $\Delta$ & Raw$\to{}$\cegp{} & $\Delta$\\
\midrule
CLIP-B/16 & MSP & $65.61\to63.58$ & \goodcell{$-2.03$} & $49.76\to46.93$ & \goodcell{$-2.83$}\\
 & MaxLogit & $43.24\to42.32$ & \goodcell{$-0.92$} & $39.90\to37.56$ & \goodcell{$-2.33$}\\
 & Energy & $58.73\to56.63$ & \goodcell{$-2.10$} & $67.38\to65.26$ & \goodcell{$-2.13$}\\
 & MCM & $58.39\to56.29$ & \goodcell{$-2.11$} & $39.35\to37.22$ & \goodcell{$-2.13$}\\
 & GL-MCM$^\dagger$ & $54.93\to52.88$ & \goodcell{$-2.04$} & $36.17\to34.35$ & \goodcell{$-1.82$}\\
\midrule
SigLIP2-B/16 & MSP & $62.95\to60.65$ & \goodcell{$-2.31$} & $56.61\to53.83$ & \goodcell{$-2.78$}\\
 & MaxLogit & $36.59\to36.40$ & \goodcell{$-0.19$} & $27.42\to27.09$ & \goodcell{$-0.33$}\\
 & Energy & $59.71\to58.27$ & \goodcell{$-1.45$} & $55.01\to53.40$ & \goodcell{$-1.62$}\\
 & MCM & $52.86\to51.44$ & \goodcell{$-1.42$} & $42.55\to41.23$ & \goodcell{$-1.32$}\\
 & GL-MCM$^\dagger$ & $54.48\to52.77$ & \goodcell{$-1.71$} & $43.43\to41.58$ & \goodcell{$-1.85$}\\
\midrule
PE-Core-B/16 & MSP & $57.65\to55.40$ & \goodcell{$-2.25$} & $40.67\to38.17$ & \goodcell{$-2.50$}\\
 & MaxLogit & $30.77\to30.92$ & \badcell{$+0.16$} & $24.83\to23.85$ & \goodcell{$-0.98$}\\
 & Energy & $55.58\to53.71$ & \goodcell{$-1.86$} & $62.86\to60.47$ & \goodcell{$-2.40$}\\
 & MCM & $46.53\to45.70$ & \goodcell{$-0.83$} & $28.02\to27.17$ & \goodcell{$-0.85$}\\
 & GL-MCM$^\dagger$ & $48.70\to47.92$ & \goodcell{$-0.78$} & $28.94\to28.38$ & \goodcell{$-0.56$}\\
\bottomrule
\end{tabular*}
\vspace{1pt}\par\tiny $^\dagger$Uses native local tokens.
\end{table}

\paragraph{Broad failure locations.}
Across all five disclosed bases, nominal unadjusted intervals classify 136
display units as improvements, 49 as inconclusive, and 10 as regressions.
Regressions concentrate in strong domain-specific bases,
notably Energy/MaxLogit on EuroSAT and MSP/MCM on DomainNet. For MCM alone,
paired suite changes are $-2.19$ [$-2.54,-1.83$] for CLIP, $-1.43$
[$-1.86,-1.02$] for SigLIP2, and $-0.92$ [$-1.33,-0.54$] for PE; its sole
resolved regression is PE DomainNet ($89.43\rightarrow89.74$).

\paragraph{Symmetric cross-fit control.}
The original broad protocol scores held-out ID with one fold-specific guard but
averages five guard scores for each OOD image. We recomputed every point
estimate foldwise, scoring held-out ID and OOD with the same guard before
averaging folds. All 15 suite deltas remain negative; relative to the reported
construction, changes lie between $-0.21$ and $0.00$ FPR95 points. The
asymmetry is therefore real and disclosed, but does not explain the broad mean
gains.

\paragraph{Trade-offs and alternatives.}
The hard guard is not a Pareto fix: Table~\ref{tab:ceg-lambda-frontier} shows a
$+16.04$ worst task regression and 55/135 regressed task cells at $\lambda=1$,
whereas \cegp{} limits the worst regression to $+1.01$ and improves 110/135
cells. Regressed residual
fusion, spatial coherence, and weighted addition hurt strong bases or traded
mean against worst case. Simes/Fisher modestly beat the minimum on the
development suite (Table~\ref{tab:ceg-ablation}) but lose its strict
intersection semantics and were not prospectively transferred. \cegp{} is the
disclosed post-result compromise: it shrinks \ceg{} toward the base, improving
transfer stability while giving up much of the minimum rule's aggregate gain.

\paragraph{Complete domain-level results.}
Tables~\ref{tab:ceg-all-0} to~\ref{tab:ceg-all-2} report FPR95 and AUROC for
Raw, CEG, and CEG-P on every one of the 13 display units and each VLM.
They use all evaluation images and preserve the post-result status of this
broad audit; the tables are descriptive rather than an independent
confirmation.

\paragraph{Unaggregated source transparency.}
Tables~\ref{tab:mcm-per-ood-1} to~\ref{tab:mcm-per-ood-5} expand the MCM base used by Figure~\ref{fig:cegarch}
into every raw ID dataset and every declared OOD source. Unlike the display-unit
tables, it does not average DomainNet styles or near/far sources. The complete
machine-readable source table additionally includes all five bases, three
variants, both FPR95 and AUROC, and all three VLMs.

At this source resolution, \cegp{} improves 896 of 1,185
base/VLM/ID/OOD point comparisons, leaves 51 unchanged, and worsens 238;
far sources improve more often and more strongly (494/645, $-1.40$ points on
average) than near sources (402/540, $-1.05$). For MCM alone the counts are
183/237 improved, six unchanged, and 48 worsened. The largest protected gains
are CLIP CIFAR-100$\to$MNIST ($-7.6$) and MNIST$\to$Places365 ($-7.1$), while
the recurring failure is DomainNet-real$\to$iNaturalist ($+2.1$ to $+3.0$)
on every VLM. This natural-to-natural failure is consistent with the guard's
boundary: if OOD remains ID-like in the base, level, and sharpness channels,
there is no missing cue for a minimum to veto. Source averaging would hide
this structured failure as apparently random noise.

\FloatBarrier
\clearpage
% AUTO-GENERATED by scripts/104_make_ceg_paper_tables.py
\begin{table*}[t]
\centering\small
\setlength{\tabcolsep}{3.6pt}
\renewcommand{\arraystretch}{0.88}
\caption{\textbf{Full all-domain CEG audit: CLIP-B/16.} Each cell is FPR95$\downarrow$/AUROC$\uparrow$ on every image; Raw is the base score, CEG is the non-compensatory minimum, and CEG-P is the protected blend. For guarded rows, green/red denotes FPR95 lower/higher than Raw by more than $0.10$ points and gray denotes a practical tie; color is keyed to FPR95 only. Values average OOD sources within a task; DomainNet averages its six styles once. This is post-result development evidence.}
\label{tab:ceg-all-0}
\textbf{Native and natural/fine-grained units}\par\vspace{1pt}
\begin{tabular*}{\textwidth}{@{\extracolsep{\fill}}llrrrrrrr}
\toprule
Base & Variant & IN-1K & C-10 & C-100 & MNIST & IN-200 & CUB & Food\\
\midrule
MSP & Raw & 66.4/78.1 & 26.0/93.7 & 86.9/75.1 & 37.3/90.1 & 36.3/90.3 & 79.0/68.1 & 20.8/94.6\\
 & CEG & \goodcell{54.0/82.3} & \goodcell{11.7/97.1} & \goodcell{57.7/85.7} & \goodcell{16.1/96.6} & \goodcell{28.7/92.2} & \goodcell{4.3/98.2} & \goodcell{5.0/97.8}\\
 & CEG-P & \goodcell{64.0/79.8} & \goodcell{22.4/94.8} & \goodcell{83.1/78.4} & \goodcell{32.4/92.1} & \goodcell{34.4/91.1} & \goodcell{74.9/75.5} & \goodcell{17.6/95.5}\\
\addlinespace[1pt]
MaxLogit & Raw & 75.7/76.9 & 14.5/96.2 & 61.6/85.7 & 23.4/95.1 & 44.8/89.5 & 3.0/98.7 & 5.2/98.1\\
 & CEG & \goodcell{54.5/82.3} & \goodcell{11.2/97.1} & \goodcell{52.7/87.5} & \goodcell{15.3/97.1} & \goodcell{30.4/92.0} & \badcell{3.2/98.5} & \goodcell{4.7/98.0}\\
 & CEG-P & \goodcell{71.5/78.7} & \goodcell{13.2/96.6} & \goodcell{59.1/86.5} & \goodcell{22.0/95.7} & \goodcell{40.2/90.4} & \neutralcell{3.0/98.6} & \goodcell{5.0/98.1}\\
\addlinespace[1pt]
Energy & Raw & 97.1/41.2 & 68.3/70.5 & 76.7/64.4 & 34.0/90.3 & 94.9/50.0 & 19.8/93.7 & 52.0/86.0\\
 & CEG & \goodcell{58.0/79.4} & \goodcell{12.8/96.8} & \goodcell{55.5/86.9} & \goodcell{15.5/97.0} & \goodcell{34.2/90.6} & \goodcell{4.1/98.2} & \goodcell{4.9/97.9}\\
 & CEG-P & \goodcell{96.7/50.7} & \goodcell{65.8/77.7} & \goodcell{75.7/71.1} & \goodcell{31.9/92.3} & \goodcell{93.3/60.5} & \goodcell{16.8/94.9} & \goodcell{44.1/89.4}\\
\addlinespace[1pt]
MCM & Raw & 60.7/81.9 & 19.3/94.7 & 82.8/76.3 & 31.5/93.4 & 34.8/90.9 & 26.1/94.1 & 8.6/97.2\\
 & CEG & \goodcell{53.9/82.9} & \goodcell{11.0/97.2} & \goodcell{57.4/85.8} & \goodcell{15.3/97.0} & \goodcell{29.5/92.2} & \goodcell{4.0/98.4} & \goodcell{4.9/97.9}\\
 & CEG-P & \goodcell{58.7/82.6} & \goodcell{16.6/95.5} & \goodcell{79.5/79.2} & \goodcell{27.2/94.5} & \goodcell{32.9/91.5} & \goodcell{21.4/95.2} & \goodcell{7.6/97.4}\\
\addlinespace[1pt]
GL-MCM$^\dagger$ & Raw & 53.4/83.2 & 16.4/95.3 & 68.1/79.8 & 31.4/92.9 & 30.2/91.6 & 49.1/85.6 & 5.6/98.0\\
 & CEG & \goodcell{52.8/82.9} & \goodcell{10.8/97.1} & \goodcell{53.7/86.4} & \goodcell{13.9/97.3} & \goodcell{28.5/92.2} & \goodcell{4.1/98.3} & \goodcell{4.8/98.0}\\
 & CEG-P & \goodcell{52.6/83.4} & \goodcell{14.7/95.9} & \goodcell{65.7/81.9} & \goodcell{28.9/94.2} & \goodcell{29.2/91.9} & \goodcell{43.8/88.8} & \goodcell{5.4/98.0}\\
\bottomrule
\end{tabular*}
\par
\vspace{5pt}
\textbf{Remaining domains and equal-unit mean}\par\vspace{1pt}
\begin{tabular*}{\textwidth}{@{\extracolsep{\fill}}llrrrrrrr}
\toprule
Base & Variant & Pets & Caltech & EuroSAT & DTD & VOC & DomainNet & Mean\\
\midrule
MSP & Raw & 37.8/91.3 & 32.6/91.2 & 99.2/31.0 & 83.3/71.4 & 47.5/84.2 & 90.0/61.8 & 57.2/78.5\\
 & CEG & \goodcell{0.5/99.7} & \goodcell{27.2/93.3} & \goodcell{65.1/67.3} & \goodcell{58.2/87.6} & \goodcell{41.1/88.9} & \goodcell{84.9/67.1} & \goodcell{35.0/88.8}\\
 & CEG-P & \goodcell{33.2/93.6} & \goodcell{30.9/92.0} & \goodcell{99.0/41.0} & \goodcell{81.1/76.2} & \neutralcell{47.5/85.7} & \goodcell{89.8/63.6} & \goodcell{54.6/81.5}\\
\addlinespace[1pt]
MaxLogit & Raw & 0.4/99.9 & 45.5/88.3 & 57.2/69.3 & 64.5/85.4 & 53.8/84.3 & 87.6/56.7 & 41.3/86.5\\
 & CEG & \neutralcell{0.4/99.7} & \goodcell{28.8/92.7} & \badcell{57.8/70.6} & \goodcell{58.1/88.1} & \goodcell{38.5/89.3} & \goodcell{84.0/66.1} & \goodcell{33.8/89.2}\\
 & CEG-P & \neutralcell{0.4/99.8} & \goodcell{42.5/89.7} & \goodcell{56.9/69.8} & \goodcell{61.5/86.7} & \goodcell{51.9/85.8} & \goodcell{87.4/59.7} & \goodcell{39.6/87.4}\\
\addlinespace[1pt]
Energy & Raw & 18.1/95.0 & 95.1/49.1 & 16.8/95.9 & 74.9/78.6 & 91.0/54.3 & 84.6/38.9 & 63.3/69.8\\
 & CEG & \goodcell{0.4/99.7} & \goodcell{33.6/91.2} & \badcell{24.1/93.5} & \goodcell{58.2/87.4} & \goodcell{43.6/87.7} & \goodcell{79.6/64.0} & \goodcell{32.7/90.0}\\
 & CEG-P & \goodcell{14.8/96.3} & \goodcell{94.0/60.3} & \badcell{17.5/95.6} & \goodcell{71.7/82.0} & \goodcell{89.4/62.9} & \neutralcell{84.5/44.7} & \goodcell{61.2/75.2}\\
\addlinespace[1pt]
MCM & Raw & 7.5/98.4 & 37.5/91.1 & 99.6/20.2 & 83.4/74.2 & 44.6/86.0 & 90.3/63.8 & 48.2/81.7\\
 & CEG & \goodcell{0.5/99.7} & \goodcell{28.0/93.1} & \goodcell{65.0/67.3} & \goodcell{57.7/88.1} & \goodcell{41.3/89.1} & \goodcell{85.1/68.6} & \goodcell{34.9/89.0}\\
 & CEG-P & \goodcell{5.7/98.7} & \goodcell{33.7/91.8} & \neutralcell{99.5/33.3} & \goodcell{80.8/78.4} & \goodcell{44.2/87.0} & \neutralcell{90.3/65.5} & \goodcell{46.0/83.9}\\
\addlinespace[1pt]
GL-MCM$^\dagger$ & Raw & 10.8/97.7 & 34.5/92.2 & 93.1/50.3 & 64.8/83.8 & 37.2/89.1 & 87.8/65.6 & 44.8/85.0\\
 & CEG & \goodcell{0.5/99.7} & \goodcell{28.2/93.2} & \goodcell{64.2/69.0} & \goodcell{53.1/88.9} & \goodcell{36.2/89.8} & \goodcell{84.6/68.4} & \goodcell{33.5/89.3}\\
 & CEG-P & \goodcell{8.5/98.2} & \goodcell{31.0/92.6} & \goodcell{91.2/55.6} & \goodcell{62.7/85.5} & \goodcell{35.2/89.4} & \badcell{88.0/66.7} & \goodcell{42.8/86.3}\\
\bottomrule
\end{tabular*}
\par
\end{table*}

\begin{table*}[t]
\centering\small
\setlength{\tabcolsep}{3.6pt}
\renewcommand{\arraystretch}{0.88}
\caption{\textbf{Full all-domain CEG audit: SigLIP2-B/16.} Each cell is FPR95$\downarrow$/AUROC$\uparrow$ on every image; Raw is the base score, CEG is the non-compensatory minimum, and CEG-P is the protected blend. For guarded rows, green/red denotes FPR95 lower/higher than Raw by more than $0.10$ points and gray denotes a practical tie; color is keyed to FPR95 only. Values average OOD sources within a task; DomainNet averages its six styles once. This is post-result development evidence.}
\label{tab:ceg-all-1}
\textbf{Native and natural/fine-grained units}\par\vspace{1pt}
\begin{tabular*}{\textwidth}{@{\extracolsep{\fill}}llrrrrrrr}
\toprule
Base & Variant & IN-1K & C-10 & C-100 & MNIST & IN-200 & CUB & Food\\
\midrule
MSP & Raw & 56.5/83.6 & 36.1/93.5 & 85.9/77.4 & 88.1/83.1 & 23.4/94.3 & 78.3/77.6 & 32.5/93.6\\
 & CEG & \badcell{58.6/84.0} & \goodcell{18.1/95.4} & \goodcell{59.3/84.2} & \goodcell{53.4/88.0} & \badcell{23.9/93.6} & \goodcell{2.9/98.8} & \goodcell{6.7/97.5}\\
 & CEG-P & \goodcell{55.9/84.4} & \goodcell{31.3/94.3} & \goodcell{83.5/79.7} & \goodcell{85.8/84.8} & \goodcell{22.4/94.4} & \goodcell{72.3/82.8} & \goodcell{27.2/94.7}\\
\addlinespace[1pt]
MaxLogit & Raw & 55.8/84.9 & 8.0/97.6 & 39.7/90.9 & 10.0/97.9 & 25.6/93.5 & 2.0/99.2 & 3.8/98.4\\
 & CEG & \badcell{58.7/84.3} & \badcell{10.4/96.9} & \badcell{44.9/88.8} & \badcell{12.5/97.5} & \goodcell{25.4/93.3} & \neutralcell{2.1/99.1} & \badcell{4.2/98.2}\\
 & CEG-P & \goodcell{55.3/85.2} & \neutralcell{8.1/97.6} & \goodcell{38.9/90.8} & \goodcell{9.5/98.0} & \goodcell{24.4/93.7} & \neutralcell{1.9/99.2} & \neutralcell{3.8/98.4}\\
\addlinespace[1pt]
Energy & Raw & 95.8/47.1 & 49.9/77.6 & 73.8/65.5 & 11.8/97.4 & 94.6/52.1 & 30.2/91.8 & 15.4/95.2\\
 & CEG & \goodcell{69.5/79.4} & \goodcell{20.8/94.7} & \goodcell{56.3/85.2} & \badcell{13.7/97.3} & \goodcell{38.3/89.6} & \goodcell{2.8/98.8} & \goodcell{6.4/97.7}\\
 & CEG-P & \goodcell{95.2/54.7} & \goodcell{47.1/82.0} & \goodcell{72.8/70.9} & \goodcell{11.4/97.6} & \goodcell{92.9/61.6} & \goodcell{26.5/93.7} & \goodcell{12.4/96.0}\\
\addlinespace[1pt]
MCM & Raw & 52.1/86.7 & 16.6/95.2 & 85.7/78.8 & 88.9/82.4 & 19.2/94.8 & 21.2/95.7 & 11.1/96.5\\
 & CEG & \badcell{58.5/84.5} & \badcell{16.8/95.7} & \goodcell{58.5/84.3} & \goodcell{53.3/87.8} & \badcell{22.6/93.9} & \goodcell{2.6/98.9} & \goodcell{6.3/97.6}\\
 & CEG-P & \goodcell{51.7/86.6} & \goodcell{14.9/95.6} & \goodcell{83.5/80.7} & \goodcell{87.2/84.1} & \goodcell{19.1/94.8} & \goodcell{16.7/96.5} & \goodcell{9.3/96.9}\\
\addlinespace[1pt]
GL-MCM$^\dagger$ & Raw & 55.8/85.5 & 25.7/94.3 & 83.0/76.6 & 83.3/82.0 & 25.7/93.5 & 27.5/94.3 & 13.5/96.2\\
 & CEG & \badcell{59.3/84.4} & \goodcell{18.2/95.4} & \goodcell{57.1/84.5} & \goodcell{52.9/87.5} & \badcell{27.3/93.0} & \goodcell{2.6/98.9} & \goodcell{6.4/97.6}\\
 & CEG-P & \goodcell{55.3/85.8} & \goodcell{20.9/94.8} & \goodcell{80.3/79.2} & \goodcell{81.3/83.7} & \goodcell{25.0/93.7} & \goodcell{22.4/95.5} & \goodcell{10.5/96.6}\\
\bottomrule
\end{tabular*}
\par
\vspace{5pt}
\textbf{Remaining domains and equal-unit mean}\par\vspace{1pt}
\begin{tabular*}{\textwidth}{@{\extracolsep{\fill}}llrrrrrrr}
\toprule
Base & Variant & Pets & Caltech & EuroSAT & DTD & VOC & DomainNet & Mean\\
\midrule
MSP & Raw & 46.2/91.9 & 19.4/94.9 & 99.5/24.5 & 70.3/81.1 & 57.2/84.4 & 83.5/67.2 & 59.8/80.5\\
 & CEG & \goodcell{0.4/99.7} & \goodcell{16.8/94.5} & \goodcell{75.1/61.3} & \goodcell{65.5/85.1} & \goodcell{45.3/86.3} & \badcell{86.3/64.9} & \goodcell{39.4/87.2}\\
 & CEG-P & \goodcell{39.9/94.0} & \goodcell{17.0/95.0} & \goodcell{99.3/34.5} & \goodcell{68.7/82.9} & \goodcell{56.1/85.3} & \badcell{83.7/67.2} & \goodcell{57.2/82.6}\\
\addlinespace[1pt]
MaxLogit & Raw & 0.1/100.0 & 10.9/96.1 & 49.1/73.0 & 71.5/83.7 & 40.1/86.7 & 87.8/55.9 & 31.1/89.1\\
 & CEG & \neutralcell{0.1/99.8} & \badcell{11.1/95.5} & \badcell{54.6/71.2} & \goodcell{68.0/85.2} & \badcell{43.6/87.1} & \goodcell{85.0/64.1} & \badcell{32.4/89.3}\\
 & CEG-P & \neutralcell{0.1/99.9} & \goodcell{10.7/96.0} & \badcell{49.7/72.5} & \goodcell{69.9/84.9} & \badcell{40.3/87.1} & \neutralcell{87.8/58.3} & \goodcell{30.8/89.3}\\
\addlinespace[1pt]
Energy & Raw & 4.1/98.7 & 89.7/68.8 & 9.3/97.9 & 90.9/63.5 & 83.2/66.8 & 84.1/33.3 & 56.4/73.5\\
 & CEG & \goodcell{0.3/99.8} & \goodcell{22.5/93.4} & \badcell{17.4/96.3} & \goodcell{73.9/81.1} & \goodcell{52.2/84.2} & \goodcell{78.9/62.0} & \goodcell{34.8/89.2}\\
 & CEG-P & \goodcell{3.3/98.9} & \goodcell{86.7/75.5} & \badcell{10.0/97.7} & \goodcell{89.8/69.1} & \goodcell{80.9/71.6} & \neutralcell{84.1/38.4} & \goodcell{54.9/77.5}\\
\addlinespace[1pt]
MCM & Raw & 7.3/98.3 & 12.5/95.6 & 99.7/13.4 & 64.6/86.1 & 50.7/86.7 & 86.1/65.9 & 47.4/82.8\\
 & CEG & \goodcell{0.3/99.8} & \badcell{13.5/95.1} & \goodcell{75.1/61.0} & \goodcell{64.4/86.2} & \goodcell{45.4/86.9} & \badcell{87.4/64.9} & \goodcell{38.8/87.4}\\
 & CEG-P & \goodcell{5.3/98.7} & \goodcell{12.0/95.6} & \neutralcell{99.6/27.2} & \goodcell{63.2/86.8} & \goodcell{49.0/87.0} & \badcell{86.3/66.1} & \goodcell{46.0/84.4}\\
\addlinespace[1pt]
GL-MCM$^\dagger$ & Raw & 14.2/97.0 & 18.7/94.3 & 99.9/11.0 & 54.0/89.0 & 43.6/87.2 & 86.2/65.6 & 48.5/82.0\\
 & CEG & \goodcell{0.3/99.8} & \goodcell{17.3/94.3} & \goodcell{74.9/61.8} & \badcell{61.3/87.6} & \badcell{44.9/86.8} & \badcell{87.5/64.2} & \goodcell{39.3/87.4}\\
 & CEG-P & \goodcell{9.9/97.8} & \goodcell{16.8/94.5} & \neutralcell{99.9/26.1} & \goodcell{53.2/89.3} & \badcell{45.1/87.3} & \badcell{86.3/65.7} & \goodcell{46.7/83.8}\\
\bottomrule
\end{tabular*}
\par
\end{table*}

\begin{table*}[t]
\centering\small
\setlength{\tabcolsep}{3.6pt}
\renewcommand{\arraystretch}{0.88}
\caption{\textbf{Full all-domain CEG audit: PE-Core-B/16.} Each cell is FPR95$\downarrow$/AUROC$\uparrow$ on every image; Raw is the base score, CEG is the non-compensatory minimum, and CEG-P is the protected blend. For guarded rows, green/red denotes FPR95 lower/higher than Raw by more than $0.10$ points and gray denotes a practical tie; color is keyed to FPR95 only. Values average OOD sources within a task; DomainNet averages its six styles once. This is post-result development evidence.}
\label{tab:ceg-all-2}
\textbf{Native and natural/fine-grained units}\par\vspace{1pt}
\begin{tabular*}{\textwidth}{@{\extracolsep{\fill}}llrrrrrrr}
\toprule
Base & Variant & IN-1K & C-10 & C-100 & MNIST & IN-200 & CUB & Food\\
\midrule
MSP & Raw & 64.6/79.4 & 13.3/96.7 & 49.0/87.2 & 48.5/85.5 & 35.4/92.1 & 38.7/90.9 & 20.3/95.0\\
 & CEG & \goodcell{53.6/84.4} & \goodcell{11.0/97.1} & \goodcell{43.3/88.8} & \goodcell{10.6/97.4} & \goodcell{32.1/92.0} & \goodcell{3.3/98.8} & \goodcell{5.7/97.8}\\
 & CEG-P & \goodcell{62.5/81.3} & \goodcell{11.7/96.9} & \goodcell{46.9/88.1} & \goodcell{44.0/88.6} & \goodcell{33.3/92.5} & \goodcell{33.6/93.0} & \goodcell{17.5/95.8}\\
\addlinespace[1pt]
MaxLogit & Raw & 60.0/84.2 & 12.2/96.8 & 32.5/93.0 & 3.0/99.2 & 28.6/93.3 & 1.8/99.3 & 3.6/98.4\\
 & CEG & \goodcell{53.3/85.3} & \goodcell{11.5/96.9} & \badcell{34.4/92.1} & \badcell{3.8/99.0} & \badcell{29.5/92.7} & \badcell{2.0/99.1} & \badcell{3.8/98.3}\\
 & CEG-P & \goodcell{58.3/85.0} & \goodcell{11.0/97.0} & \goodcell{31.2/93.1} & \goodcell{2.9/99.2} & \goodcell{27.6/93.4} & \neutralcell{1.8/99.2} & \neutralcell{3.6/98.4}\\
\addlinespace[1pt]
Energy & Raw & 97.2/37.0 & 69.0/54.3 & 75.9/53.8 & 3.2/99.3 & 95.9/38.9 & 44.0/90.2 & 15.5/95.6\\
 & CEG & \goodcell{59.4/80.2} & \goodcell{18.2/95.3} & \goodcell{50.1/86.7} & \badcell{3.4/99.1} & \goodcell{44.8/88.4} & \goodcell{3.5/98.8} & \goodcell{5.0/98.0}\\
 & CEG-P & \goodcell{97.0/46.7} & \goodcell{66.4/63.0} & \goodcell{74.5/60.7} & \goodcell{2.9/99.3} & \goodcell{95.2/49.8} & \goodcell{36.3/92.6} & \goodcell{12.7/96.3}\\
\addlinespace[1pt]
MCM & Raw & 51.1/85.7 & 8.8/97.4 & 35.3/91.5 & 37.5/89.0 & 25.0/93.6 & 4.9/98.4 & 7.0/97.3\\
 & CEG & \goodcell{50.7/85.7} & \badcell{9.2/97.3} & \badcell{37.0/90.6} & \goodcell{9.5/97.7} & \badcell{27.6/92.8} & \goodcell{2.5/98.9} & \goodcell{5.2/97.9}\\
 & CEG-P & \goodcell{50.2/86.1} & \goodcell{8.6/97.4} & \goodcell{34.7/91.6} & \goodcell{33.3/91.3} & \goodcell{24.6/93.6} & \goodcell{4.1/98.5} & \goodcell{6.6/97.5}\\
\addlinespace[1pt]
GL-MCM$^\dagger$ & Raw & 48.4/86.4 & 9.7/97.2 & 39.9/90.3 & 46.9/86.4 & 24.2/93.6 & 5.4/98.2 & 7.9/97.0\\
 & CEG & \badcell{50.4/85.6} & \badcell{10.1/97.1} & \badcell{40.1/89.7} & \goodcell{11.6/97.2} & \badcell{28.7/92.6} & \goodcell{2.7/98.9} & \goodcell{5.3/97.9}\\
 & CEG-P & \goodcell{48.1/86.5} & \goodcell{9.5/97.2} & \goodcell{39.3/90.4} & \goodcell{42.7/89.3} & \neutralcell{24.2/93.6} & \goodcell{4.6/98.4} & \goodcell{7.3/97.3}\\
\bottomrule
\end{tabular*}
\par
\vspace{5pt}
\textbf{Remaining domains and equal-unit mean}\par\vspace{1pt}
\begin{tabular*}{\textwidth}{@{\extracolsep{\fill}}llrrrrrrr}
\toprule
Base & Variant & Pets & Caltech & EuroSAT & DTD & VOC & DomainNet & Mean\\
\midrule
MSP & Raw & 25.2/95.9 & 32.5/91.9 & 98.7/40.4 & 63.4/84.4 & 47.6/86.2 & 87.9/62.3 & 48.1/83.7\\
 & CEG & \goodcell{0.4/99.7} & \goodcell{26.1/92.4} & \goodcell{37.6/87.0} & \goodcell{61.2/86.8} & \goodcell{39.9/88.5} & \goodcell{84.4/66.4} & \goodcell{31.5/90.5}\\
 & CEG-P & \goodcell{19.4/96.9} & \goodcell{30.3/92.4} & \goodcell{98.1/52.6} & \goodcell{63.1/85.6} & \goodcell{45.1/87.1} & \badcell{88.2/63.7} & \goodcell{45.7/85.7}\\
\addlinespace[1pt]
MaxLogit & Raw & 0.3/99.9 & 12.4/95.8 & 37.2/89.2 & 35.0/93.0 & 42.6/87.2 & 83.6/66.7 & 27.1/92.0\\
 & CEG & \neutralcell{0.2/99.8} & \badcell{18.4/94.5} & \goodcell{35.8/87.9} & \badcell{49.2/90.6} & \goodcell{41.8/88.6} & \goodcell{82.7/66.7} & \badcell{28.2/91.6}\\
 & CEG-P & \neutralcell{0.3/99.9} & \badcell{13.1/95.6} & \goodcell{35.7/89.2} & \neutralcell{35.1/93.0} & \goodcell{42.5/87.9} & \goodcell{83.4/67.0} & \goodcell{26.7/92.1}\\
\addlinespace[1pt]
Energy & Raw & 39.8/87.5 & 67.3/74.2 & 4.2/99.0 & 74.8/76.5 & 96.0/51.6 & 81.6/47.1 & 58.8/69.6\\
 & CEG & \goodcell{0.4/99.7} & \goodcell{30.2/92.6} & \badcell{17.9/97.5} & \goodcell{65.0/85.3} & \goodcell{46.9/84.4} & \goodcell{78.0/63.5} & \goodcell{32.5/90.0}\\
 & CEG-P & \goodcell{33.6/90.7} & \goodcell{63.4/79.2} & \badcell{4.5/99.0} & \goodcell{73.4/80.1} & \goodcell{95.3/59.9} & \neutralcell{81.6/50.5} & \goodcell{56.7/74.4}\\
\addlinespace[1pt]
MCM & Raw & 1.0/99.8 & 26.7/93.1 & 98.9/27.8 & 45.6/89.9 & 37.9/88.4 & 89.4/61.7 & 36.1/85.7\\
 & CEG & \goodcell{0.3/99.8} & \goodcell{26.2/93.0} & \goodcell{37.4/87.1} & \badcell{54.1/88.9} & \badcell{38.7/89.2} & \goodcell{85.1/66.9} & \goodcell{29.5/91.2}\\
 & CEG-P & \goodcell{0.8/99.7} & \goodcell{24.7/93.3} & \goodcell{98.5/43.7} & \goodcell{45.4/90.1} & \goodcell{36.5/88.9} & \badcell{89.7/63.4} & \goodcell{35.2/87.3}\\
\addlinespace[1pt]
GL-MCM$^\dagger$ & Raw & 1.0/99.7 & 26.5/93.1 & 99.7/21.8 & 55.4/87.8 & 34.7/89.2 & 89.1/61.5 & 37.6/84.8\\
 & CEG & \goodcell{0.3/99.8} & \badcell{27.4/92.8} & \goodcell{36.8/87.7} & \badcell{57.4/87.7} & \badcell{39.4/89.1} & \goodcell{85.0/66.6} & \goodcell{30.4/91.0}\\
 & CEG-P & \goodcell{0.8/99.7} & \goodcell{24.9/93.3} & \neutralcell{99.6/40.3} & \goodcell{54.7/88.2} & \neutralcell{34.6/89.5} & \badcell{89.4/63.1} & \goodcell{36.9/86.7}\\
\bottomrule
\end{tabular*}
\par
\end{table*}

\clearpage

% AUTO-GENERATED by scripts/104_make_ceg_paper_tables.py
\begin{table*}[p]
\centering\scriptsize
\setlength{\tabcolsep}{7pt}
\renewcommand{\arraystretch}{0.94}
\caption{\textbf{Unaggregated ID$\times$OOD-source MCM results (part 1/5).} Each VLM cell is Raw/CEG/CEG-P FPR95$\downarrow$ (\%). Guarded numbers are green/red when lower/higher than Raw by more than $0.10$ points and gray for a practical tie. No OOD source or DomainNet style is averaged. Guarded variants average five ID-only calibration seeds.}
\label{tab:mcm-per-ood-1}
\begin{tabular*}{\textwidth}{@{\extracolsep{\fill}}lllrrr@{}}
\toprule
ID dataset & OOD source & Type & CLIP-B/16 & SigLIP2-B/16 & PE-Core-B/16\\
\midrule
IN-1K & SSB-hard & Near & 88.7/\goodnum{85.5}/\goodnum{88.1} & 79.2/\badnum{82.3}/\neutralnum{79.2} & 85.3/\goodnum{84.0}/\goodnum{85.1}\\
 & NINCO & Near & 79.9/\goodnum{72.0}/\goodnum{78.7} & 65.8/\badnum{70.3}/\goodnum{65.4} & 71.0/\goodnum{69.0}/\goodnum{70.1}\\
 & iNaturalist & Far & 31.9/\goodnum{16.0}/\goodnum{27.3} & 30.5/\badnum{42.0}/\goodnum{29.8} & 28.0/\badnum{30.1}/\goodnum{26.9}\\
 & Textures & Far & 59.5/\badnum{59.9}/\goodnum{58.5} & 54.1/\badnum{57.9}/\goodnum{53.8} & 42.7/\goodnum{39.6}/\goodnum{41.1}\\
 & OpenImage-O & Far & 43.7/\goodnum{36.1}/\goodnum{41.0} & 30.8/\badnum{40.1}/\goodnum{30.4} & 28.5/\badnum{30.8}/\goodnum{27.7}\\
\addlinespace[1.5pt]
CIFAR-10 & CIFAR-100 & Near & 35.6/\goodnum{29.7}/\goodnum{33.3} & 31.8/\badnum{32.8}/\goodnum{30.0} & 18.2/\badnum{23.4}/\badnum{18.6}\\
 & TinyImageNet & Near & 34.4/\goodnum{16.3}/\goodnum{29.5} & 28.4/\badnum{30.7}/\goodnum{27.6} & 20.1/\badnum{22.1}/\neutralnum{20.2}\\
 & MNIST & Far & 0.8/\badnum{1.0}/\neutralnum{0.7} & 5.3/\badnum{11.9}/\goodnum{3.5} & 0.1/\neutralnum{0.1}/\neutralnum{0.1}\\
 & SVHN & Far & 1.6/\badnum{2.2}/\neutralnum{1.7} & 1.2/\badnum{4.9}/\goodnum{1.0} & 0.1/\neutralnum{0.2}/\neutralnum{0.1}\\
 & Places365 & Far & 29.6/\goodnum{11.1}/\goodnum{24.0} & 25.5/\goodnum{17.8}/\goodnum{22.4} & 11.6/\goodnum{8.1}/\goodnum{10.8}\\
 & Textures & Far & 13.8/\goodnum{5.6}/\goodnum{10.7} & 7.4/\goodnum{2.8}/\goodnum{4.9} & 2.7/\goodnum{1.4}/\goodnum{2.2}\\
\addlinespace[1.5pt]
CIFAR-100 & CIFAR-10 & Near & 88.7/\goodnum{82.8}/\goodnum{88.0} & 87.3/\goodnum{63.7}/\goodnum{84.5} & 52.2/\badnum{56.8}/\goodnum{51.5}\\
 & TinyImageNet & Near & 93.3/\goodnum{64.8}/\goodnum{91.7} & 88.1/\goodnum{73.4}/\goodnum{86.2} & 63.4/\badnum{68.1}/\neutralnum{63.4}\\
 & MNIST & Far & 61.8/\goodnum{38.6}/\goodnum{54.2} & 100.0/\goodnum{54.7}/\neutralnum{99.9} & 8.2/\badnum{11.4}/\goodnum{8.0}\\
 & SVHN & Far & 63.4/\goodnum{41.4}/\goodnum{57.6} & 46.4/\badnum{49.6}/\goodnum{43.7} & 0.8/\badnum{2.0}/\neutralnum{0.8}\\
 & Places365 & Far & 97.8/\goodnum{57.3}/\goodnum{96.3} & 98.6/\goodnum{71.2}/\goodnum{97.6} & 55.3/\goodnum{49.0}/\goodnum{53.3}\\
 & Textures & Far & 92.0/\goodnum{59.6}/\goodnum{89.4} & 94.1/\goodnum{38.6}/\goodnum{89.3} & 31.7/\badnum{35.0}/\goodnum{30.9}\\
\bottomrule
\end{tabular*}
\end{table*}

\begin{table*}[p]
\centering\scriptsize
\setlength{\tabcolsep}{7pt}
\renewcommand{\arraystretch}{0.94}
\caption{\textbf{Unaggregated ID$\times$OOD-source MCM results (part 2/5).} Each VLM cell is Raw/CEG/CEG-P FPR95$\downarrow$ (\%). Guarded numbers are green/red when lower/higher than Raw by more than $0.10$ points and gray for a practical tie. No OOD source or DomainNet style is averaged. Guarded variants average five ID-only calibration seeds.}
\label{tab:mcm-per-ood-2}
\begin{tabular*}{\textwidth}{@{\extracolsep{\fill}}lllrrr@{}}
\toprule
ID dataset & OOD source & Type & CLIP-B/16 & SigLIP2-B/16 & PE-Core-B/16\\
\midrule
MNIST & notMNIST & Near & 10.3/\goodnum{8.8}/\goodnum{9.4} & 80.1/\goodnum{64.2}/\goodnum{78.5} & 60.8/\goodnum{32.8}/\goodnum{57.6}\\
 & FashionMNIST & Near & 41.9/\goodnum{31.4}/\goodnum{39.1} & 75.3/\goodnum{71.7}/\goodnum{73.8} & 37.0/\goodnum{3.3}/\goodnum{32.7}\\
 & Textures & Far & 32.4/\goodnum{17.4}/\goodnum{26.9} & 93.7/\goodnum{24.3}/\goodnum{90.7} & 31.0/\goodnum{6.6}/\goodnum{26.0}\\
 & CIFAR-10 & Far & 28.0/\goodnum{19.6}/\goodnum{25.0} & 96.6/\goodnum{49.4}/\goodnum{95.5} & 23.9/\goodnum{5.4}/\goodnum{20.7}\\
 & TinyImageNet & Far & 30.1/\goodnum{4.6}/\goodnum{23.4} & 90.5/\goodnum{52.2}/\goodnum{88.5} & 28.4/\goodnum{5.5}/\goodnum{24.7}\\
 & Places365 & Far & 46.4/\goodnum{10.0}/\goodnum{39.2} & 97.2/\goodnum{58.0}/\goodnum{96.4} & 43.7/\goodnum{3.1}/\goodnum{38.4}\\
\addlinespace[1.5pt]
IN-200 & SSB-hard & Near & 61.1/\goodnum{53.6}/\goodnum{58.7} & 37.1/\badnum{41.5}/\goodnum{36.9} & 48.2/\badnum{49.6}/\goodnum{47.7}\\
 & NINCO & Near & 53.3/\goodnum{46.3}/\goodnum{51.4} & 29.8/\badnum{35.2}/\neutralnum{29.8} & 38.5/\badnum{41.6}/\goodnum{38.0}\\
 & iNaturalist & Far & 6.0/\goodnum{3.5}/\goodnum{4.8} & 3.2/\badnum{5.5}/\neutralnum{3.3} & 4.9/\badnum{8.7}/\neutralnum{4.8}\\
 & OpenImage-O & Far & 18.7/\goodnum{14.7}/\goodnum{16.7} & 6.7/\badnum{8.4}/\goodnum{6.3} & 8.3/\badnum{10.3}/\goodnum{8.0}\\
\addlinespace[1.5pt]
CUB-200 & SSB-hard & Near & 47.2/\goodnum{14.6}/\goodnum{41.7} & 36.3/\goodnum{9.5}/\goodnum{31.0} & 16.5/\goodnum{9.3}/\goodnum{14.6}\\
 & NINCO & Near & 29.9/\goodnum{0.3}/\goodnum{23.7} & 19.9/\goodnum{0.3}/\goodnum{14.8} & 1.8/\goodnum{0.3}/\goodnum{1.1}\\
 & iNaturalist & Far & 10.5/\goodnum{0.5}/\goodnum{7.5} & 6.3/\goodnum{0.2}/\goodnum{4.1} & 0.1/\badnum{0.3}/\neutralnum{0.1}\\
 & OpenImage-O & Far & 16.9/\goodnum{0.5}/\goodnum{12.9} & 22.5/\goodnum{0.3}/\goodnum{16.9} & 1.0/\goodnum{0.2}/\goodnum{0.7}\\
\addlinespace[1.5pt]
Food-101 & SSB-hard & Near & 12.9/\goodnum{6.1}/\goodnum{11.7} & 14.3/\goodnum{6.3}/\goodnum{12.1} & 11.4/\goodnum{6.2}/\goodnum{10.7}\\
 & NINCO & Near & 16.6/\goodnum{11.1}/\goodnum{14.7} & 22.8/\goodnum{14.7}/\goodnum{19.4} & 13.1/\goodnum{12.0}/\goodnum{12.5}\\
 & iNaturalist & Far & 0.9/\goodnum{0.2}/\goodnum{0.7} & 2.8/\goodnum{1.6}/\goodnum{2.0} & 0.8/\goodnum{0.7}/\neutralnum{0.8}\\
 & OpenImage-O & Far & 4.0/\goodnum{2.2}/\goodnum{3.4} & 4.5/\goodnum{2.5}/\goodnum{3.6} & 2.8/\goodnum{1.8}/\goodnum{2.4}\\
\bottomrule
\end{tabular*}
\end{table*}

\begin{table*}[p]
\centering\scriptsize
\setlength{\tabcolsep}{7pt}
\renewcommand{\arraystretch}{0.94}
\caption{\textbf{Unaggregated ID$\times$OOD-source MCM results (part 3/5).} Each VLM cell is Raw/CEG/CEG-P FPR95$\downarrow$ (\%). Guarded numbers are green/red when lower/higher than Raw by more than $0.10$ points and gray for a practical tie. No OOD source or DomainNet style is averaged. Guarded variants average five ID-only calibration seeds.}
\label{tab:mcm-per-ood-3}
\begin{tabular*}{\textwidth}{@{\extracolsep{\fill}}lllrrr@{}}
\toprule
ID dataset & OOD source & Type & CLIP-B/16 & SigLIP2-B/16 & PE-Core-B/16\\
\midrule
Pets & SSB-hard & Near & 15.0/\goodnum{1.2}/\goodnum{12.3} & 8.7/\goodnum{1.0}/\goodnum{6.8} & 2.4/\goodnum{0.7}/\goodnum{1.8}\\
 & NINCO & Near & 8.9/\goodnum{0.6}/\goodnum{6.9} & 3.8/\goodnum{0.2}/\goodnum{2.3} & 1.6/\goodnum{0.3}/\goodnum{1.3}\\
 & iNaturalist & Far & 2.0/\goodnum{0.0}/\goodnum{1.1} & 7.8/\goodnum{0.0}/\goodnum{5.7} & 0.0/\neutralnum{0.0}/\neutralnum{0.0}\\
 & OpenImage-O & Far & 4.0/\goodnum{0.1}/\goodnum{2.8} & 8.9/\goodnum{0.1}/\goodnum{6.6} & 0.1/\neutralnum{0.0}/\neutralnum{0.1}\\
\addlinespace[1.5pt]
Caltech-101 & SSB-hard & Near & 60.0/\goodnum{42.9}/\goodnum{54.4} & 16.8/\badnum{17.4}/\goodnum{16.1} & 36.9/\goodnum{29.3}/\goodnum{33.2}\\
 & NINCO & Near & 54.9/\goodnum{43.5}/\goodnum{50.8} & 26.2/\badnum{27.4}/\goodnum{25.4} & 44.6/\badnum{45.6}/\goodnum{42.9}\\
 & iNaturalist & Far & 10.3/\goodnum{7.8}/\goodnum{8.3} & 1.7/\badnum{3.2}/\goodnum{1.6} & 11.3/\badnum{14.4}/\goodnum{10.0}\\
 & OpenImage-O & Far & 24.8/\goodnum{17.8}/\goodnum{21.1} & 5.2/\badnum{6.1}/\goodnum{4.9} & 13.9/\badnum{15.4}/\goodnum{12.8}\\
\addlinespace[1.5pt]
EuroSAT & SSB-hard & Near & 99.8/\goodnum{54.7}/\neutralnum{99.7} & 99.8/\goodnum{59.2}/\neutralnum{99.8} & 98.9/\goodnum{10.9}/\goodnum{98.4}\\
 & NINCO & Near & 99.8/\goodnum{44.2}/\neutralnum{99.7} & 99.5/\goodnum{67.1}/\goodnum{99.3} & 98.7/\goodnum{22.7}/\goodnum{98.1}\\
 & iNaturalist & Far & 100.0/\goodnum{99.0}/\neutralnum{100.0} & 100.0/\goodnum{99.8}/\neutralnum{100.0} & 100.0/\goodnum{79.0}/\neutralnum{100.0}\\
 & OpenImage-O & Far & 98.9/\goodnum{62.3}/\goodnum{98.6} & 99.4/\goodnum{74.4}/\neutralnum{99.3} & 98.0/\goodnum{37.0}/\goodnum{97.4}\\
\addlinespace[1.5pt]
DTD & SSB-hard & Near & 89.7/\goodnum{54.0}/\goodnum{87.0} & 64.4/\goodnum{62.7}/\goodnum{62.8} & 50.2/\badnum{55.0}/\goodnum{49.8}\\
 & NINCO & Near & 86.7/\goodnum{60.4}/\goodnum{84.6} & 67.6/\goodnum{65.8}/\goodnum{66.4} & 54.2/\badnum{64.7}/\badnum{54.7}\\
 & iNaturalist & Far & 76.0/\goodnum{55.0}/\goodnum{72.8} & 66.2/\badnum{69.6}/\goodnum{65.0} & 31.1/\badnum{45.2}/\neutralnum{31.1}\\
 & OpenImage-O & Far & 81.3/\goodnum{61.6}/\goodnum{78.7} & 60.3/\goodnum{59.6}/\goodnum{58.7} & 46.9/\badnum{51.5}/\goodnum{45.9}\\
\bottomrule
\end{tabular*}
\end{table*}

\begin{table*}[p]
\centering\scriptsize
\setlength{\tabcolsep}{7pt}
\renewcommand{\arraystretch}{0.94}
\caption{\textbf{Unaggregated ID$\times$OOD-source MCM results (part 4/5).} Each VLM cell is Raw/CEG/CEG-P FPR95$\downarrow$ (\%). Guarded numbers are green/red when lower/higher than Raw by more than $0.10$ points and gray for a practical tie. No OOD source or DomainNet style is averaged. Guarded variants average five ID-only calibration seeds.}
\label{tab:mcm-per-ood-4}
\begin{tabular*}{\textwidth}{@{\extracolsep{\fill}}lllrrr@{}}
\toprule
ID dataset & OOD source & Type & CLIP-B/16 & SigLIP2-B/16 & PE-Core-B/16\\
\midrule
VOC-13 & SSB-hard & Near & 76.8/\goodnum{73.1}/\goodnum{76.4} & 72.2/\goodnum{67.7}/\goodnum{71.2} & 67.8/\goodnum{67.1}/\goodnum{67.3}\\
 & NINCO & Near & 48.5/\goodnum{43.9}/\goodnum{48.3} & 49.3/\goodnum{41.1}/\goodnum{47.9} & 35.6/\badnum{40.0}/\goodnum{35.2}\\
 & iNaturalist & Far & 12.7/\goodnum{10.9}/\goodnum{11.9} & 43.3/\goodnum{41.9}/\goodnum{41.3} & 20.2/\badnum{20.5}/\goodnum{17.4}\\
 & OpenImage-O & Far & 40.4/\goodnum{37.3}/\goodnum{40.1} & 37.8/\goodnum{30.9}/\goodnum{35.4} & 28.0/\goodnum{27.1}/\goodnum{26.0}\\
\addlinespace[1.5pt]
DN-real & SSB-hard & Near & 83.2/\badnum{87.1}/\badnum{83.6} & 69.6/\badnum{76.7}/\badnum{70.2} & 81.5/\badnum{83.1}/\neutralnum{81.5}\\
 & NINCO & Near & 72.2/\badnum{76.5}/\neutralnum{72.3} & 54.9/\badnum{67.0}/\badnum{56.1} & 68.6/\badnum{77.6}/\badnum{69.5}\\
 & iNaturalist & Far & 60.0/\badnum{77.1}/\badnum{62.0} & 79.6/\badnum{92.1}/\badnum{81.8} & 71.6/\badnum{87.8}/\badnum{74.6}\\
 & OpenImage-O & Far & 65.6/\badnum{73.0}/\badnum{66.2} & 54.1/\badnum{62.9}/\badnum{54.7} & 60.8/\badnum{68.6}/\badnum{61.3}\\
\addlinespace[1.5pt]
DN-clipart & SSB-hard & Near & 95.6/\goodnum{92.9}/\goodnum{95.2} & 82.8/\badnum{89.8}/\badnum{83.2} & 88.9/\goodnum{88.8}/\neutralnum{89.0}\\
 & NINCO & Near & 90.8/\goodnum{88.0}/\goodnum{90.4} & 74.5/\badnum{85.1}/\badnum{75.0} & 80.9/\badnum{87.0}/\badnum{82.0}\\
 & iNaturalist & Far & 93.6/\badnum{94.2}/\badnum{93.8} & 97.2/\badnum{98.9}/\badnum{97.6} & 91.0/\badnum{95.4}/\badnum{92.2}\\
 & OpenImage-O & Far & 88.4/\goodnum{86.2}/\goodnum{88.0} & 71.7/\badnum{80.6}/\badnum{72.1} & 74.4/\badnum{78.8}/\badnum{74.8}\\
\addlinespace[1.5pt]
DN-painting & SSB-hard & Near & 96.9/\goodnum{93.3}/\goodnum{96.3} & 89.5/\badnum{92.3}/\badnum{89.7} & 94.1/\neutralnum{94.0}/\neutralnum{94.1}\\
 & NINCO & Near & 92.7/\goodnum{88.0}/\goodnum{92.2} & 85.8/\badnum{88.3}/\neutralnum{85.8} & 90.6/\badnum{92.4}/\badnum{90.8}\\
 & iNaturalist & Far & 96.0/\goodnum{94.2}/\neutralnum{95.9} & 99.2/\badnum{99.4}/\neutralnum{99.2} & 98.2/\neutralnum{98.2}/\badnum{98.3}\\
 & OpenImage-O & Far & 90.9/\goodnum{86.6}/\goodnum{90.2} & 80.6/\badnum{84.2}/\badnum{80.8} & 84.1/\badnum{86.3}/\badnum{84.3}\\
\bottomrule
\end{tabular*}
\end{table*}

\begin{table*}[p]
\centering\scriptsize
\setlength{\tabcolsep}{7pt}
\renewcommand{\arraystretch}{0.94}
\caption{\textbf{Unaggregated ID$\times$OOD-source MCM results (part 5/5).} Each VLM cell is Raw/CEG/CEG-P FPR95$\downarrow$ (\%). Guarded numbers are green/red when lower/higher than Raw by more than $0.10$ points and gray for a practical tie. No OOD source or DomainNet style is averaged. Guarded variants average five ID-only calibration seeds.}
\label{tab:mcm-per-ood-5}
\begin{tabular*}{\textwidth}{@{\extracolsep{\fill}}lllrrr@{}}
\toprule
ID dataset & OOD source & Type & CLIP-B/16 & SigLIP2-B/16 & PE-Core-B/16\\
\midrule
DN-sketch & SSB-hard & Near & 96.5/\goodnum{94.4}/\goodnum{96.1} & 90.8/\badnum{94.0}/\badnum{91.3} & 95.2/\goodnum{94.5}/\neutralnum{95.3}\\
 & NINCO & Near & 92.1/\goodnum{89.6}/\goodnum{91.7} & 87.9/\badnum{89.9}/\neutralnum{87.9} & 92.5/\badnum{93.0}/\badnum{92.9}\\
 & iNaturalist & Far & 95.2/\badnum{95.4}/\badnum{95.4} & 99.4/\neutralnum{99.5}/\neutralnum{99.5} & 98.9/\goodnum{98.4}/\badnum{99.0}\\
 & OpenImage-O & Far & 90.1/\goodnum{87.8}/\goodnum{89.8} & 82.6/\badnum{86.2}/\badnum{82.7} & 86.4/\badnum{87.7}/\badnum{86.7}\\
\addlinespace[1.5pt]
DN-quickdraw & SSB-hard & Near & 99.7/\goodnum{70.0}/\goodnum{99.6} & 96.5/\goodnum{82.5}/\goodnum{96.0} & 99.2/\goodnum{62.2}/\goodnum{99.0}\\
 & NINCO & Near & 97.6/\goodnum{63.0}/\goodnum{97.3} & 95.2/\goodnum{81.1}/\goodnum{94.6} & 98.4/\goodnum{69.9}/\goodnum{98.1}\\
 & iNaturalist & Far & 99.7/\goodnum{62.4}/\goodnum{99.6} & 99.9/\goodnum{95.0}/\neutralnum{99.9} & 100.0/\goodnum{64.1}/\neutralnum{100.0}\\
 & OpenImage-O & Far & 97.9/\goodnum{62.3}/\goodnum{97.4} & 90.9/\goodnum{76.0}/\goodnum{90.1} & 96.4/\goodnum{55.6}/\goodnum{95.8}\\
\addlinespace[1.5pt]
DN-infograph & SSB-hard & Near & 96.2/\badnum{97.4}/\neutralnum{96.2} & 96.8/\goodnum{96.2}/\neutralnum{96.7} & 99.4/\goodnum{96.0}/\goodnum{99.3}\\
 & NINCO & Near & 91.6/\badnum{93.8}/\badnum{91.9} & 95.5/\goodnum{92.0}/\goodnum{95.1} & 98.6/\goodnum{94.7}/\goodnum{98.5}\\
 & iNaturalist & Far & 94.7/\badnum{97.7}/\badnum{95.2} & 99.9/\goodnum{99.5}/\neutralnum{99.9} & 100.0/\goodnum{98.5}/\neutralnum{100.0}\\
 & OpenImage-O & Far & 89.5/\badnum{92.1}/\badnum{89.8} & 91.3/\goodnum{89.5}/\goodnum{91.0} & 96.9/\goodnum{91.2}/\goodnum{96.6}\\
\bottomrule
\end{tabular*}
\end{table*}

\clearpage

% AAAI requires the completed reproducibility checklist after the paper.
% Fill checklist/ReproducibilityChecklist.tex, then uncomment the next line.
% \input{checklist/ReproducibilityChecklist}

\end{document}